\documentclass[english]{article}
\pdfoutput=1

\newcommand{\ARXIV}{} % comment out to produce shorter camera-ready

\usepackage{times}
\usepackage{graphicx} % more modern
\usepackage{subfigure}

\usepackage{natbib}

\usepackage{algorithm}
\usepackage{algorithmic}

\usepackage{hyperref}

\usepackage[accepted]{icml2017}

\usepackage{enumerate}
\usepackage{amsthm}
\usepackage{amssymb}
\usepackage{amsmath}

\theoremstyle{plain}
  \newtheorem{thm}{\protect\theoremname}
\theoremstyle{definition}
  \newtheorem{defn}[thm]{\protect\definitionname}
  
  \newtheorem{prop}[thm]{\protect\propositionname}
  \newtheorem{cor}[thm]{\protect\corollaryname}
\theoremstyle{plain}
  \newtheorem{rem}[thm]{\protect\remarkname}
  \newtheorem{ex}[thm]{\protect\examplename}

\usepackage{babel}
  \providecommand{\definitionname}{Definition}
  \providecommand{\lemmaname}{Lemma}
  \providecommand{\propositionname}{Proposition}
  \providecommand{\corollaryname}{Corollary}
  \providecommand{\remarkname}{Remark}
  \providecommand{\examplename}{Example}
\providecommand{\theoremname}{Theorem}

\usepackage{xspace}
\usepackage[inline]{enumitem}

\newcommand{\term}[1]{\emph{#1}\xspace}
\newcommand{\cf}[0]{\emph{cf.}\xspace}
\newcommand{\eg}[0]{\emph{e.g.},\xspace}

\renewcommand{\v}[1]{\ensuremath{\mathbf{#1}}}
\renewcommand{\Pr}[1]{\ensuremath{\mathrm{Pr}\left(#1\right)}}
\newcommand{\Prs}[2]{\ensuremath{\mathrm{Pr}_{#1}\left(#2\right)}}

\newcommand{\CExp}[2]{\ensuremath{\mathbb{E}\left[\left.#1\right|#2\right]}}
\newcommand{\indic}[1]{\ensuremath{\mathbf{1}\left[#1\right]}}
\newcommand{\cdf}[1]{\ensuremath{\Phi\left(#1\right)}}
\newcommand{\ecdf}[2]{\ensuremath{\Phi_{#1}\left(#2\right)}}
\newcommand{\reals}{\ensuremath{\mathbb{R}}\xspace}
\newcommand{\naturals}{\ensuremath{\mathbb{N}}\xspace}
\newcommand{\domain}{\ensuremath{\mathcal{D}}\xspace}
\newcommand{\DBs}{\ensuremath{\domain^n}\xspace}
\newcommand{\responses}{\ensuremath{\mathcal{R}}\xspace}
\newcommand{\normed}{\ensuremath{\mathcal{B}}\xspace}
\newcommand{\norm}[1]{\ensuremath{\left\|{#1}\right\|_{\normed}}\xspace}
\newcommand{\Lap}[1]{\ensuremath{\mathrm{Lap}\left(#1\right)}\xspace}
\newcommand{\Gauss}[1]{\ensuremath{\mathcal{N}\left(#1\right)}\xspace}
\newcommand{\diag}[1]{\ensuremath{\mathrm{diag}\left(#1\right)}}
\newcommand{\sampler}{\textsc{SensitivitySampler}\xspace}
\newcommand{\samplermech}{\textsc{Sample-Then-Respond}\xspace}
\newcommand{\myparagraph}[1]{\noindent\textbf{#1}\xspace}

\icmltitlerunning{Pain-Free Random Differential Privacy with Sensitivity Sampling}

\begin{document}

\twocolumn[
\icmltitle{Pain-Free Random Differential Privacy with Sensitivity Sampling}

% It is OKAY to include author information, even for blind
% submissions: the style file will automatically remove it for you
% unless you've provided the [accepted] option to the icml2017
% package.

% list of affiliations. the first argument should be a (short)
% identifier you will use later to specify author affiliations
% Academic affiliations should list Department, University, City, Region, Country
% Industry affiliations should list Company, City, Region, Country

% you can specify symbols, otherwise they are numbered in order
% ideally, you should not use this facility. affiliations will be numbered
% in order of appearance and this is the preferred way.
%\icmlsetsymbol{equal}{*}

\begin{icmlauthorlist}
\icmlauthor{Benjamin I. P. Rubinstein}{unimelb}
\icmlauthor{Francesco Ald\`a}{bochum}
\end{icmlauthorlist}

\icmlaffiliation{unimelb}{School of Computing and Information Systems, University of Melbourne, Australia}
\icmlaffiliation{bochum}{Horst G\"ortz Institute for IT Security and Faculty of Mathematics, Ruhr-Universit\"at Bochum, Germany}

\icmlcorrespondingauthor{BR}{brubinstein@unimelb.edu.au}
\icmlcorrespondingauthor{FA}{francesco.alda@rub.de}

% You may provide any keywords that you
% find helpful for describing your paper; these are used to populate
% the "keywords" metadata in the PDF but will not be shown in the document
\icmlkeywords{Differential Privacy, Sampling}

\vskip 0.3in
]

% this must go after the closing bracket ] following \twocolumn[ ...

% This command actually creates the footnote in the first column
% listing the affiliations and the copyright notice.
% The command takes one argument, which is text to display at the start of the footnote.
% The \icmlEqualContribution command is standard text for equal contribution.
% Remove it (just {}) if you do not need this facility.

\printAffiliationsAndNotice{}  % leave blank if no need to mention equal contribution
%\printAffiliationsAndNotice{\icmlEqualContribution} % otherwise use the standard text.
%\footnotetext{hi}

\begin{abstract}
Popular approaches to differential privacy, such as the
Laplace and exponential mechanisms, calibrate randomised smoothing
through global sensitivity of the target non-private
function. Bounding such sensitivity is often a prohibitively complex
analytic calculation. As an alternative, we propose a straightforward
sampler for estimating sensitivity of non-private mechanisms. Since our
sensitivity estimates hold with high probability, any mechanism that
would be $(\epsilon,\delta)$-differentially private under bounded 
global sensitivity automatically achieves $(\epsilon,\delta,\gamma)$-random
differential privacy~\citep{hall2012random}, without any target-specific
calculations required. 
We demonstrate on worked example learners how our usable approach adopts
a naturally-relaxed privacy guarantee, while achieving more accurate
releases even for non-private functions that are black-box computer programs.
\end{abstract}

\section{Introduction}

Differential privacy \cite{dwork2006calibrating} has emerged as the
dominant framework for protected 
privacy of sensitive training data when releasing learned models to untrusted
third parties. This paradigm owes its popularity in part to the strong privacy model
provided, and in part to the availability of general building block mechanisms such as
the Laplace \cite{dwork2006calibrating} \& exponential \cite{mcsherry2007mechanism},
and to composition lemmas for
building up more complex mechanisms. These generic
mechanisms come endowed with privacy and utility bounds that hold
for any appropriate application. Such tools almost alleviate the burden of
performing theoretical analysis in developing privacy-preserving learners. However a 
persistent requirement is the need to bound global sensitivity---a Lipschitz
constant of the target, non-private function. For
simple scalar statistics of the private database, sensitivity can be easily
bounded \cite{dwork2006calibrating}. However in many applications---from
collaborative filtering~\cite{mcsherry2009differentially} to Bayesian
inference~\cite{dimitrakakis2014robust,dimitrakakis2014robustNew,wang2015privacy}---the principal challenge
in privatisation is completing this calculation.

In this work we develop a simple approach to approximating
global sensitivity with high probability, assuming only oracle access
to target function evaluations.
Combined with generic mechanisms like Laplace, exponential, Gaussian or Bernstein, our sampler
enables systematising of privatisation: arbitrary computer programs
can be made differentially private with no additional mathematical analysis nor
dynamic/static analysis, whatsoever.
Our approach does not make any assumptions about the function under evaluation
or underlying sampling distribution. 

\myparagraph{Contributions.} This paper contributes: 
\begin{enumerate*}[label=\roman*)]
	\item \sampler for easily-implemented empirical estimation of global sensitivity of (potentially black-box) non-private mechanisms;
	\item Empirical process theory for guaranteeing random differential privacy for any mechanism that preserves (stronger) differential privacy under bounded global sensitivity;
	\item Experiments demonstrating our sampler on learners for which analytical sensitivity bounds are highly involved; and %Future mechanisms may employ our sampler instead of performing such analysis; and
	\item Examples where sensitivity estimates beat (pessimistic) bounds, delivering pain-free random differential privacy at higher levels of accuracy, when used in concert with generic privacy-preserving mechanisms.
\end{enumerate*}

%We next review related work found in the machine learning, databases and theory literature. In Sections~\ref{sec:background} \&~\ref{sec:problem} we recall preliminary definitions and develop our setting \& algorithm for sensitivity sampling in Sections~\ref{sec:sensitivity-induced} \&~\ref{sec:algorithm}. We then provide guarantees for our sampler in Section~\ref{sec:theory} and experimental evaluation in Section~\ref{sec:experiments}. Section~\ref{sec:conclusion} concludes the paper.

\myparagraph{Related Work.}
This paper builds on the large body of work in differential
privacy~\cite{dwork2006calibrating,dwork2014algorithmic}, which has gained broad
interest in part due the framework's strong guarantees of data privacy when releasing aggregate
statistics or models, and due to availability of many generic privatising mechanisms
\eg: Laplace~\cite{dwork2006calibrating}, 
exponential~\cite{mcsherry2007mechanism}, Gaussian~\cite{dwork2014algorithmic},
Bernstein~\cite{alda2017bernstein} and many more. 
While these mechanisms present a path to privatisation without
need for reproving differential privacy or utility, they do have in common
a need to analytically bound sensitivity---a Lipschitz-type condition on the target
non-private function. Often derivations are 
intricate \eg for collaborative filtering~\cite{mcsherry2009differentially},
SVMs~\cite{rubinstein2012learning,chaudhuri2011differentially},
model selection~\cite{thakurta2013differentially}, feature selection~\cite{kifer2012private},
Bayesian inference~\cite{dimitrakakis2014robust,dimitrakakis2014robustNew,wang2015privacy}, SGD in deep
learning~\cite{abadi2016deep}, etc. Undoubtedly the non-trivial
nature of bounding sensitivity prohibits adoption by some
domain experts. We address this challenge through the
\sampler that estimates sensitivity empirically---even for privatising black-box
computer programs---providing high probability privacy guarantees generically.

Several systems have been developed to ease deployment of
differentially privacy, with \citet{barthe2016programming} overviewing
contributions from Programming Languages.
Dynamic approaches track privacy budget expended at runtime, 
\ifdefined\ARXIV typically through basic operations on data with known privacy loss, \fi
\eg the PINQ~\cite{mcsherry2009privacy,mcsherry2010differentially} and
Airavat~\cite{roy2010airavat} systems. 
\ifdefined\ARXIV These create a C\# LINQ-like interface and a framework for bringing differential privacy to MapReduce, respectively. 
While such approaches are flexible, the lack of static checking means privacy
violations are not caught until after (lengthy) computations. Static checking
approaches provide forewarning of privacy usage. In this mould,
Fuzz~\cite{reed2010distance,palamidessi} offers a higher-order functional
language whose static type system tracks sensitivity based on linear logic,
so that sensitivity (and therefore differential privacy) is guaranteed by
typechecking. Limited to a similar level of program expressiveness as PINQ,
Fuzz cannot typecheck many desirable target programs and specifically cannot
track data-dependent function sensitivity, with the DFuzz system demonstrating
preliminary work towards this challenge~\cite{gaboardi2013linear}. \else
Static checking approaches provide privacy usage forewarning, \eg: 
Fuzz~\cite{reed2010distance,palamidessi}, DFuzz~\cite{gaboardi2013linear}. 
\fi
While promoting privacy-by-design, such approaches impose 
specialised languages or limit target feasibility---challenges
addressed by this work. Moreover our \sampler mechanism 
complements such systems, \eg within broader frameworks for
protecting against side-channel attacks~\cite{haeberlen2011differential,mohan2012gupt}.

\citet{NIPS2016_6050} show that special-case Gibbs sampler is $(\epsilon,\delta)$-DP
without bounded sensitivity. 
\citet{nissim2007smooth} ask: Why calibrate for worst-case
global sensitivity when the \emph{actual} database does not witness worst-case neighbours?
Their smoothed sensitivity approach privatises local sensitivity, which itself is sensitive to perturbation. While this can lead to better sensitivity estimates, our sampled
sensitivity still does not require analytical
bounds. A related approach is the sample-and-aggregate
mechanism~\cite{nissim2007smooth} which avoids computation of sensitivity of
the underlying target function and instead requires sensitivity of an aggregator
combining the outputs of the non-private target run repeatedly on subsamples of the
data. By contrast, our approach provides direct sensitivity estimates,
permitting direct privatisation.

Our application of empirical process theory to estimate hard-to-compute quantities
resembles the work of \citet{riondato2015mining}. They use VC-theory and sampling
to approximate mining frequent itemsets.
Here we approximate analytical computations, and to our knowledge provide a first
generic mechanism that preserves random differential privacy~\cite{hall2012random}---a
natural weakening of the strong guarantee of differential privacy. \citet{hall2012random}
leverage empirical process theory for a specific worked example, while our setting is general sensitivity estimation.

\section{Background}
\label{sec:background}

We are interested in non-private mechanism $f:\DBs\to\normed$
that maps \term{databases} in product space over domain \domain
to \term{responses} in a normed space \normed. The terminology of
``database" (DB) comes from statistical databases, and should be understood
as a dataset.

\begin{ex}
For instance, in supervised learning of linear classifiers, the domain could
be Euclidean vectors comprising features \& labels, and responses
might be parameterisations of learned classifiers such as a normal
vector.
\end{ex}

We aim to
estimate sensitivity which is commonly used to calibrate noise
in differentially-private mechanisms.
\begin{defn}
	The \term{global sensitivity} of non-private $f:\DBs\to\normed$
	is given by $\overline{\Delta}=\sup_{D,D'}\norm{f(D)-f(D')}$,
where the supremum is taken over all pairs of \term{neighbouring} databases $D,D'$ in
$\domain^n$ that differ in one point.
\end{defn}

%\begin{rem}
%Note that the generality of response space may appear to be overkill:
%Laplace, Gaussian and Bernstein mechanisms all require sensitivity
%over Euclidean space, under the $L_1$ or $L_2$ norms. However, for the
%exponential mechanism it is convenient to consider the score
%function itself as the release, with function space norm.
%Section~\ref{sec:sensitivity-induced} presents detailed accounts of each mechanism.
%\end{rem}

\begin{defn}
	Randomized \term{mechanism} $M:\DBs\to\responses$ responding
	with values in arbitrary \term{response set} \responses preserves \term{$\epsilon$-differential}
privacy for $\epsilon>0$ if for all neighbouring $D,D'\in\domain^n$
and measurable $R\subset\responses$ it holds that $\Pr{M(D)\in R}\leq\exp(\epsilon)\Pr{M(D')\in R}$.
If instead for $0<\delta<1$ it holds that $\Pr{M(D)\in R}\leq\exp(\epsilon)\Pr{M(D')\in R}+\delta$
then the mechanism preserves the weaker notion of \term{$(\epsilon,\delta)$-differential
privacy.}
\end{defn}

In Section~\ref{sec:sensitivity-induced}, we recall a number of key mechanisms that preserve these notions of
privacy by virtue of target non-private function sensitivity.

%Typical choices on Euclidean
%responses include the $L_{1}$ and $L_{2}$ norms.

%$P$ could be the underlying distribution from which a sensitive database
%was drawn---in the case of sensitive training data but insensitive sources
%of test data; $P$ could be an alternate test distribution of interest in
%the case of domain adaptation; a uniform distribution or otherwise
%non-informative likelihood.
%
%We recall a number of mechanisms that satisfy

%We will focus on the Laplace mechanism for the remainder.
%\begin{lem}
%For $f:\DBs\to\reals^{d}$ the Laplace mechanism
%releases $f(D)$ with added iid Laplace noise of mean zero and scale
%$\Delta_{n}/\epsilon$ and achieves $\epsilon$-differential privacy.
%\end{lem}

The following definition due to \citet{hall2012random} relaxes the requirement that
uniform smoothness of response distribution holds on all pairs of databases, to the
requirement that uniform smoothness holds for likely database pairs.

\begin{defn}\label{def:random}
	Randomized mechanism $M:\DBs\to\responses$ \linebreak[4] responding
	with values in an arbitrary response set \responses \linebreak[4] preserves $(\epsilon,\gamma)$-random
	differential privacy, at \linebreak[4] privacy level $\epsilon>0$ and confidence $\gamma\in(0,1)$, if
$\Pr{\forall R\subset\responses, \Pr{M(D)\in R}\leq e^{\epsilon}\Pr{M(D')\in R}} \geq  1-\gamma$,
with the inner probabilities over the mechanism's randomization, and the
outer probability over neighbouring $D,D'\in\domain^n$
drawn from some $P^{n+1}$. The weaker $(\epsilon,\delta)$-DP has analogous definition as $(\epsilon,\delta,\gamma)$-RDP.
\end{defn}

\begin{rem}
While strong $\epsilon$-DP is ideal, utility may
demand compromise. Precedent exists for weaker privacy,
with the definition of $(\epsilon,\delta)$-DP wherein on any databases (including likely ones)
a private mechanism may leak sensitive information on low probability responses, 
forgiven by the additive $\delta$ relaxation. $(\epsilon,\gamma)$-RDP offers an alternate
relaxation, where on all but a small $\gamma$-proportion of unlikely database
pairs, strong $\epsilon$-DP holds---RDP plays a useful role. %The next example demonstrates another view. % on RDP vs DP.}
\end{rem}

\begin{ex}\label{ex:sample-mean}
Consider a database on unbounded positive reals $D\in\reals_+^n$ representing
loan default times of a bank's customers, and target release statistic 
$f(D)=n^{-1}\sum_{i=1}^n D_i$ the sample mean. 
To $\epsilon$-DP privatise scalar-valued $f(D)$ it is natural to look to the Laplace mechanism.
However the mechanism requires a bound on the statistic's global sensitivity,
impossible under unbounded $D$. Note 
for $\Delta>0$, when neighbouring  $D, D'$ satisfy $\{|f(D)-f(D')|\leq\Delta\}$
then Laplace mechanism $M_{\Delta,\epsilon}(f(D))$ enjoys $\epsilon$-DP on that
DB pair. 
%$\{\forall t\in\reals, \Pr{M_{\Delta(D),\epsilon}=t}\leq \exp(\epsilon)\Pr{M_{\Delta,\epsilon}(D')=t}\}$.
Therefore the probability of the latter event is bounded below by the probability of the former.
Modelling the default times by iid exponential variables of rate $\lambda>0$,
then $|f(D)-f(D')|=|D_n-D_n'|/n$ is 
distributed as $\mathrm{Exp}(n\lambda)$, and so 
\begin{align*}
& \Pr{\forall t\in\reals, \Pr{M_{\Delta,\epsilon}(D)=t}\leq e^\epsilon \Pr{M_{\Delta,\epsilon}(D')=t}} \\
\geq & \Pr{|f(D)-f(D')|\leq\Delta} 
 =  1 - e^{-\lambda n \Delta} \geq  1 - \gamma\enspace,
\end{align*}
provided that $\Delta\geq \log(1/\gamma)/(\lambda n)$. While
$\epsilon$-DP fails due to unboundedness, the data is likely
bounded and so the mechanism is likely strongly private: 
$M_{\Delta,\epsilon}$ is $(\epsilon,\gamma)$-RDP. 
\end{ex}

\section{Problem Statement}
\label{sec:problem}

We consider a statistician looking to apply a differentially-private
mechanism to an $f:\DBs\to\normed$ whose sensitivity cannot easily be
bounded analytically (\cf Example~\ref{ex:sample-mean} or the case of a computer program).

%\ifdefined\ARXIV
%\begin{ex}
%The statistician may wish to release $f$ implicitly defined as a computer
%program, using the $\epsilon$-differentially-private Laplace mechanism
%(\cf Corollary~\ref{cor:laplace})
%$M_{\overline{\Delta}}(D)\sim \mathrm{Lap}\left(f(D), \overline{\Delta}/\epsilon\right)$. However
%the sensitivity $\overline{\Delta}$ may not be forthcoming.
%\end{ex}
%\fi

Instead we assume that the statistician has the ability to sample from some arbitrary
product space $P^{n+1}$ on $\domain^{n+1}$, can evaluate $f$
arbitrarily (and in particular on the result of this sampling), and
is interested in applying a privatising mechanism with the guarantee
of random differential privacy (Definition~\ref{def:random}).

\begin{rem}\label{rem:P}
Natural
choices for $P$ present themselves for sampling or defining random
differential privacy. $P$ could be taken as the underlying
distribution from which a sensitive DB was drawn---in the case of
sensitive training data but insensitive data source;
an alternate test distribution of interest in the case of domain
adaptation;  or $P$ could be uniform or an otherwise
non-informative likelihood (\cf Example~\ref{ex:sample-mean}). Proved in 
\ifdefined\ARXIV Appendix~\ref{sec:transferrdp-proof}, \else full report~\cite{AR2017}, \fi
the following relates RDP of similar distributions.
\end{rem}

\begin{prop}\label{prop:transferRDP}
Let $P, Q$ be distributions on \domain with bounded KL divergence $KL(P\|Q)\leq\tau$.
If mechanism $M$ on databases in $\domain^n$ is RDP with confidence $\gamma>0$ wrt $P$ then
it is also RDP with confidence $\gamma+\sqrt{(n+1)\tau/2}$ wrt $Q$, with the same
privacy parameters $\epsilon$ (or $\epsilon,\delta$).
\end{prop}

\section{Sensitivity-Induced Differential Privacy}
\label{sec:sensitivity-induced}

When a privatising mechanism $M$ is known to achieve
differential privacy for some mapping $f:\DBs\to\normed$ under bounded
global sensitivity, then our approach's
high-probability estimates of sensitivity will imply high-probability
preservation of differential privacy. In order to reason about such
arguments, we introduce the concept of sensitivity-induced differential
privacy.

\begin{defn}
For arbitrary mapping $f:\DBs\to\normed$ and randomised mechanism
$M_\Delta: \normed\to\responses$, we say that \term{$M_\Delta$ is sensitivity-induced
$\epsilon$-differentially private} if for a neighbouring pair of
databases $D,D'\in\domain^n$, and $\Delta\geq 0$
\begin{align*}
	&\norm{f(D) - f(D')} \leq \Delta  \\
	\Longrightarrow \hspace{1em}
	&\forall R\subset\responses,\ \Pr{M_\Delta(f(D))\in R} \\
	& \hphantom{\forall R\subset\responses,\ \ }\leq \ \exp(\epsilon) \cdot \Pr{M_\Delta(f(D'))\in R}
\end{align*}
with the qualification on $R$ being all measurable subsets of the response set \responses.
In the same vein, the analogous definition for $(\epsilon,\delta)$-differential privacy
can also be made.
\end{defn}

Many generic mechanisms in use today preserve differential privacy by virtue of
satisfying this condition. The following are immediate consequences of existing
proofs of differential privacy. First, when a non-private target function
$f$ aims to release Euclidean vectors responses.

\begin{cor}[Laplace mechanism]\label{cor:laplace}
Consider database $D\in\DBs$, normed space
$\normed=(\reals^d,\|\cdot\|_1)$ for $d\in\naturals$, non-private function
$f: \DBs \to \normed$.
The Laplace mechanism\footnote{$\Lap{\v{a},b}$ has unnormalised PDF
$\exp(-\|\v{x}-\v{a}\|_1/b)$.} \cite{dwork2006calibrating}
$M_\Delta(f(D)) \sim \Lap{f(D), \Delta/\epsilon}$, is sensitivity-induced
$\epsilon$-differentially private.
\end{cor}

\begin{ex}
Example~\ref{ex:sample-mean} used Corollary~\ref{cor:laplace} for
RDP of the Laplace mechanism on unbounded bank loan defaults.
\end{ex}

\begin{cor}[Gaussian mechanism]\label{cor:gaussian}
Consider database $D\in\DBs$, normed space
$\normed=(\reals^d,\|\cdot\|_2)$ for some $d\in\naturals$, and non-private function
$f: \DBs \to \normed$.
The Gaussian mechanism \cite{dwork2014algorithmic}
$M_\Delta(f(D)) \sim \Gauss{f(D), \diag{\sigma}}$ with
$\sigma^2>2\Delta^2 \log(1.25/\delta)/\epsilon^2$, is sensitivity-induced
$(\epsilon,\delta)$-differentially private.
\end{cor}

Second, $f$ may aim to release elements of an arbitrary set $\responses$,
where a score function $s(D,\cdot)$ benchmarks quality of potential releases (placing a partial ordering on $\responses$).

\begin{cor}[Exponential mechanism]\label{cor:exponential}
Consider database $D\in\DBs$, response space \responses, normed space
$\normed=\left(\reals^\responses, \|\cdot\|_\infty\right)$, non-private score function
$s: \DBs\times\responses\to\reals$, and restriction
$f: \DBs \to \normed$ given by $f(D)=s(D,\cdot)$.
The exponential mechanism \cite{mcsherry2007mechanism}
$M_\Delta(f(D)) \sim \exp\left(\epsilon \left(f(D)\right)(r) / 2\Delta\right)$, which
when normalised specifies a PDF over responses $r\in\responses$, is sensitivity-induced
$\epsilon$-differentially private.
\end{cor}

Third, $f$ could be function-valued as for learning settings, where given a training
set we wish to release a model (\eg classifier or predictive posterior) that
can be subsequently evaluated on (non-sensitive) test points.

\begin{cor}[Bernstein mechanism]\label{cor:bernstein}
Consider database $D\in\DBs$, query space $\mathcal{Y}=[0,1]^\ell$ with constant
dimension $\ell\in\naturals$, lattice cover of $\mathcal{Y}$ of size $k\in\naturals$
given by $\mathcal{L}=\left(\left\{0,1/k,\ldots,1\right\}\right)^\ell$, normed space
$\normed=\left(\reals^{\mathcal{Y}},\|\cdot\|_\infty\right)$, non-private function
$F: \DBs\times\mathcal{Y}\to\reals$, and restriction
$f: \DBs\to\normed$ given by $f(D)=F(D,\cdot)$. The Bernstein
mechanism~\cite{alda2017bernstein} $M_\Delta(f(D))\sim \left\{ \mathrm{Lap} \left((f(D))(\v{p}), \Delta(k+1)^{\ell}/\epsilon \right) \mid\v{p} \in \mathcal{L}\right\}$, is
sensitivity-induced $\epsilon$-differentially private.
\end{cor}

Our framework does not apply directly to the objective perturbation mechanism
of \citet{chaudhuri2011differentially}, as that mechanism does not rely directly
on a notion of sensitivity of objective function, classifier, or otherwise. However
it can apply to the posterior
sampler used for differentially-private Bayesian
inference~\cite{mir2012differentially,dimitrakakis2014robust,dimitrakakis2014robustNew,zhang2016differential}:
there the target function $f: \DBs\to\normed$ returns the likelihood function
$p(D|\cdot)$, itself mapping parameters $\Theta$ to $\reals$; using the
result of $f(D)$ and public prior $\xi(\theta)$, the mechanism samples from the
posterior $\xi(B|D)=\int_B p(D|\theta)d\xi(\theta) / \int_\Theta p(D|\theta)d\xi(\theta)$;
differential privacy follows from a Lipschitz condition on $f$ that would require
our sensitivity sampler to sample from all database pairs---a minor modification
left for future work.

\begin{algorithm}[tb]
\caption{\sampler}
   \label{algo:sampler}
\begin{algorithmic}
	\STATE {\bfseries Input:}
	database size $n$,
	target mapping $f:\DBs\to\normed$,
%	RDP confidence $\gamma$,
	sample size $m$,
	order statistic index $k$,
%	approximation confidence $\rho$,
	distribution $P$
   \FOR{$i=1$ {\bfseries to} $m$}
   \STATE Sample $D \sim P^{n+1}$
   \STATE Set $G_i = \norm{f\left(D_{1\ldots n}\right) - f\left(D_{1\ldots n-1, n+1}\right) }$
   \ENDFOR
   \STATE Sort $G_1,\ldots,G_m$ as $G_{(1)}\leq\ldots\leq G_{(m)}$
   %\STATE Set $\rho = \exp\left(W_{-1}\left(-\frac{\gamma}{2\sqrt{e}}\right)+\frac{1}{2}\right)$
   %\STATE Set $k = \left\lceil m\left( 1 -\gamma +\rho + \sqrt{\log(1/\rho) / (2m)} \right) \right\rceil$
   \STATE {\bfseries return} $\hat{\Delta} = G_{(k)}$
\end{algorithmic}
\end{algorithm}

\section{The Sensitivity Sampler}
\label{sec:algorithm}

Algorithm~\ref{algo:sampler} presents the \sampler in detail.
Consider privacy-insensitive independent sample $D_{1},\ldots,D_{m}\sim P^{n+1}$
of databases on $n+1$ records, where $P$ is chosen to match the desired
distribution in definition of random differential privacy. A number of natural
choices are available for $P$ (\cf Remark~\ref{rem:P}). The main idea of \sampler
is that for each extended-database observation of $D\sim P^{n+1}$, we induce
i.i.d. observations $G_{1},\ldots,G_{m}\in\reals$ of the random variable
\begin{eqnarray*}
	G & = & \norm{f\left(D_{1\ldots n}\right)-f\left(D_{1\ldots n-1;n+1}\right)} \enspace.
\end{eqnarray*}
From these observations of the sensitivity of target mapping $f:\DBs\to\normed$,
we estimate w.h.p. sensitivity that can achieve random differential
privacy, for the full suite of sensitivity-induced private mechanisms
discussed above.

\begin{algorithm}[t!]
	\caption{\samplermech}
	\label{algo:mechanism}
	\begin{algorithmic}
		\STATE {\bfseries Input:}
		database $D$;
		randomised mechanism $M_\Delta: \normed\to\responses$;
		target mapping $f:\DBs\to\normed$,
%		RDP confidence $\gamma$;
		sample size $m$,
		order statistic index $k$,
%		approximation confidence $\rho$,
		distribution $P$
		%\STATE Set $\hat{\Delta}$ to \sampler$(|D|, f, \gamma, m, k, \rho, P)$
		\STATE Set $\hat{\Delta}$ to \sampler$(|D|, f, m, k, P)$
		\STATE {\bfseries respond} $M_{\hat{\Delta}}(D)$
	\end{algorithmic}
\end{algorithm}

If we knew the full CDF of $G$, we would simply invert
this CDF to determine the level of sensitivity for achieving any desired $\gamma$ level of
random differential privacy: higher confidence would invoke higher sensitivity
and therefore lower utility. However as we cannot in general
possess the true CDF, we resort to uniformly approximating it w.h.p. using the
empirical CDF induced by the sample $G_1,\ldots,G_m$. The guarantee of uniform
approximation derives from empirical process theory. Figure~\ref{fig:samplerhowto}
provides further intuition behind \sampler. 
Algorithm~\ref{algo:mechanism} presents \samplermech which composes \sampler with
any sensitivity-induced differentially-private mechanism. 

\begin{figure}[b!]
\begin{center}
\vspace{-1em}
\centerline{\includegraphics[width=0.80\columnwidth]{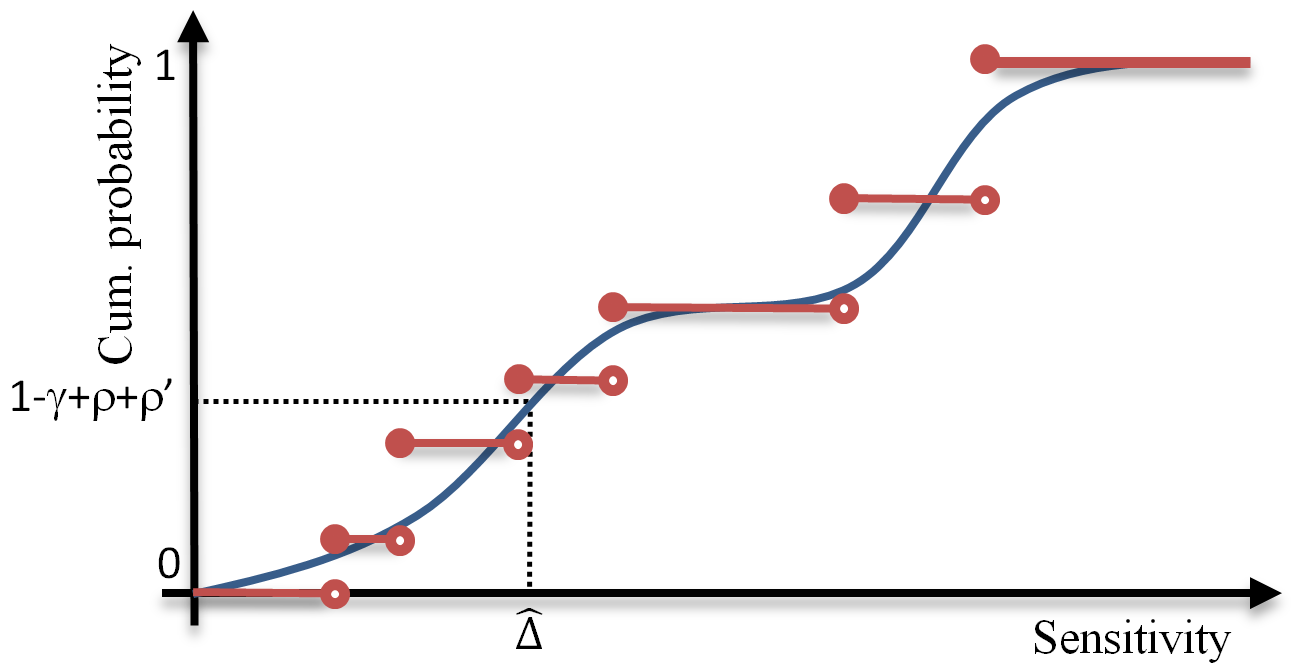}}
\caption{Inside \sampler: the true sensitivity CDF (blue); empirical sensitivity CDF (piecewise constant red); inversion of the empirical CDF (black dotted); where $\rho,\rho'$ are the DKW confidence and errors defined in Theorem~\ref{thm:main-rdp}.}
\label{fig:samplerhowto}
\vspace{-1em}
\end{center}
\end{figure}

Our main result Theorem~\ref{thm:main-rdp} presents explicit expressions for parameters
$m, k$ that are sufficient to guarantee that \samplermech achieves
$(\epsilon,\delta,\gamma)$-random differential privacy. Under that result the parameter
$\rho$, which controls the uniform approximation of the empirical CDF from
$G_1,\ldots,G_m$ sample to the true CDF, is introduced as a free parameter.
We demonstrate through a series of optimisations in 
\ifdefined\ARXIV Corollaries~\ref{cor:min-m}--\ref{cor:min-k} \else Table~\ref{tab:tuning} \fi
how $\rho$ can be tuned to optimise either sampling effort $m$, utility via order statistic
index $k$, or privacy confidence $\gamma$. These alternative explicit choices for $\rho$
serve as optimal operating points for the mechanism.

\subsection{Practicalities}

\sampler simplifies the application of differential privacy by obviating
the challenge of bounding sensitivity. As such, it is important to
explore any practical issues arising in its implementation. 
The algorithm itself involves few main stages: sampling databases, measuring sensitivity,
sorting, order statistic lookup (inversion), followed by the sensitivity-induced private
mechanism.

\myparagraph{Sampling.} As discussed in Remark~\ref{rem:P}, a number of natural choices
for sampling distribution $P$ could be made. Where a simulation process exists, capable
of generating synthetic data approximating $D$, then this could be run.
For example in the Bayesian setting~\cite{dimitrakakis2014robust}, one could use a public
conditional likelihood $p(\cdot|\theta)$, parametric family $\Theta$, prior $\xi(\theta)$
and sample from the marginal $\int_\Theta p(x|\theta)d\xi(\theta)$. Alternatively, it may
suffice to sample from the uniform distribution on \domain, or Gaussian restricted to
Euclidean \domain. In any of these cases, sampling is relatively straightforward and the
choice should consider meaningful random differential privacy guarantees relative to $P$.

\myparagraph{Sensitivity Measurement.} A trivial stage, given neighbouring databases,
measurement could involve expanding a mathematical expression representing a target
function, or a computer program such as running a deep learning or computer vision
open-source package. For some targets, it may be that running first on one database,
covers much of the computation required for the neighbouring database in which case
amortisation may improve runtime. The cost of sensitivity measurement will be primarily
determined by sample size $m$. Note that sampling and measurement
can be trivially parallelised over map-reduce-like platforms.

\myparagraph{Sorting, Inversion.} Strictly speaking the entire sensitivity sample
need not be sorted, as only one order statistic is required. That said, sorting
even millions of scalar measurements can be accomplished in under a second on a stock
machine. An alternative strategy to inversion as presented, is to take the \emph{maximum
sensitivity measured} so as to maximise privacy without consideration to utility.

\myparagraph{Mechanism.} It is noteworthy that in settings where mechanism $M_\Delta$
is to be run multiple times, the estimation of $\hat{\Delta}$ need not be redone.
As such \sampler could be performed entirely in an offline amortisation stage.

\begin{table*}[t]
	\caption{Optimal $\rho$ operating points for budgeted resources---$\gamma$ or $m$---minimising $m$, $\gamma$ or $k$; proved in \ifdefined\ARXIV Appendix~\ref{app:optimising}. \else \cite{AR2017}. \fi \label{tab:tuning}}
\begin{tabular}{cccccc}
\hline
Budgeted & Optimise & $\rho$ & $\gamma$ & $m$ & $k$ \\
\hline
$\gamma\in(0,1)$ & $m$ & $\exp\left(W_{-1}\left(-\frac{\gamma}{2\sqrt{e}}\right)+\frac{1}{2}\right)$ & $\bullet$ & $\left\lceil\frac{\log\left(\frac{1}{\rho}\right)}{2(\gamma-\rho)^2}\right\rceil$ & $\left\lceil m\left( 1 -\gamma +\rho + \sqrt{\frac{\log\left(\frac{1}{\rho}\right)}{2m}} \right)\right\rceil$ \\
$m\in\naturals, \gamma$ & $k$ & $\exp\left(\frac{1}{2}W_{-1}\left(-\frac{1}{4m}\right)\right)$ & $\geq \rho + \sqrt{\frac{\log\left(\frac{1}{\rho}\right)}{2m}}$ & $\bullet$ & $\left\lceil m\left( 1 -\gamma +\rho + \sqrt{\frac{\log\left(\frac{1}{\rho}\right)}{2m}} \right)\right\rceil$ \\
$m\in\naturals$ & $\gamma$ & $\exp\left(\frac{1}{2}W_{-1}\left(-\frac{1}{4m}\right)\right)$ & $\rho + \sqrt{\frac{\log\left(\frac{1}{\rho}\right)}{2m}}$ & $\bullet$ & $m$ \\
\hline
\end{tabular}
\end{table*}

\section{Analysis}
\label{sec:theory}

For the i.i.d. sample of sensitivities $G_1,\ldots,G_m$ drawn within Algorithm~\ref{algo:sampler},
denote the corresponding fixed unknown CDF, and corresponding random empirical CDF, by
\begin{eqnarray*}
\cdf{g} &=& \Pr{G\leq g}\enspace, \\
\ecdf{m}{g} &=& \frac{1}{m}\sum_{i=1}^{m}\indic{G_{i}\leq g}\enspace.
\end{eqnarray*}

In this section we use \ecdf{m}{\Delta} to bound the likelihood of a (non-private,
possibly deterministic) mapping $f:\DBs\to\responses$ achieving 
sensitivity $\Delta$. This permits bounding RDP.

\begin{thm}\label{thm:main-rdp}
Consider any non-private mapping $f: \DBs\to\normed$, any sensitivity-induced
$(\epsilon,\delta)$-differentially private mechanism $M_\Delta$ mapping $\normed$
to (randomised) responses in \responses, any database $D$ of $n$ records,
privacy parameters $\epsilon>0$, $\delta\in[0,1]$, $\gamma\in(0,1)$, and sampling
parameters size $m\in\naturals$, order statistic index $m\geq k\in\naturals$,
approximation confidence $0<\rho<\min\{\gamma, 1/2\}$, distribution $P$ on $\domain$. If
\begin{eqnarray}
	m &\geq& \frac{1}{2(\gamma-\rho)^2} \log\left(\frac{1}{\rho}\right)\enspace, \label{thm:main-rdp-condition} \\
	k &\geq& m\left( 1 -\gamma +\rho + \sqrt{\log(1/\rho) / (2m)} \right) \enspace, \label{thm:main-rdp-condition-k}
\end{eqnarray}
then Algorithm~\ref{algo:mechanism} run with
$D, M_\Delta, f, m, k, P$, preserves
$(\epsilon,\delta,\gamma)$-random differential privacy.
\end{thm}

\begin{proof}
Consider any $\rho'\in(0,1)$ to be determined later, and consider sampling
$G_1,\ldots,G_m$ and sorting to $G_{(1)}\leq\ldots\leq G_{(m)}$. Provided
that
\begin{eqnarray}
	1 - \gamma +\rho + \rho' \leq 1 &\Leftrightarrow& \rho' \leq \gamma - \rho \enspace, \label{eq:thm-pf-feasible}
\end{eqnarray}
then the random sensitivity $\hat{\Delta}=G_{(k)}$, where
$k=\lceil m(1-\gamma+\rho+\rho')\rceil$, is the smallest $\Delta\geq 0$ such that
$\Phi_m(\Delta)\geq 1-\gamma+\rho+\rho'$. That is,
\begin{eqnarray}
	\Phi_m(\hat{\Delta}) &\geq& 1-\gamma+\rho+\rho'\enspace. \label{eq:thm-pf-invert}
\end{eqnarray}
Note that if $1-\gamma+\rho+\rho'<0$ then $\hat{\Delta}$ can be taken as any $\Delta$,
namely zero.
Define the events
\begin{eqnarray*}
	A_\Delta &=  &\left\{\forall R\subset\responses,\ \Pr{M_\Delta(f(D))\in R} \leq \exp(\epsilon) \cdot \right. \\
			&& \hspace{7em} \left. \Pr{M_\Delta(f(D'))\in R} + \delta\right\} \\
	B_{\rho'} &= & \left\{ \sup_{\Delta} \left(\Phi_m(\Delta) - \Phi(\Delta)\right) \leq \rho' \right\}\enspace.
\end{eqnarray*}
The first is the event that DP holds for a specific DB pair,
when the mechanism is run with (possibly random) sensitivity parameter $\Delta$;
the second records the empirical CDF uniformly one-sided approximating the CDF to
level $\rho'$.
By the sensitivity-induced $\epsilon$-differential privacy of $M_\Delta$,
\begin{eqnarray}
	\forall\Delta>0\ , && \Prs{D,D'\sim P^{n+1}}{A_\Delta}\ \geq\ \Phi(\Delta)\enspace. \label{eq:thm-pf-CDF}
\end{eqnarray}
The random $D,D'$ on the left-hand side induce the distribution on $G$ on the right-hand
side under which $\Phi(\Delta)=\Prs{G}{G\leq\Delta}$. The probability on the left
is the level of random differential privacy of $M_\Delta$ when run on fixed $\Delta$.
By the Dvoretzky-Kiefer-Wolfowitz inequality~\cite{dkw} we have that for all
$\rho'\geq \sqrt{(\log 2) / (2m)}$,
\begin{align}
	\Prs{G_1,\ldots,G_m}{B_{\rho'}}
\geq 1 - e^{-2 m \rho'^2}\enspace. \label{eq:thm-pf-dkw}
\end{align}
Putting inequalities~\eqref{eq:thm-pf-invert}, \eqref{eq:thm-pf-CDF}, and~\eqref{eq:thm-pf-dkw}
together, provided that $\rho'\geq \sqrt{(\log 2) / (2m)}$,
yields that
\begin{figure}[t]
\centering
\centerline{\includegraphics[width=1.00\columnwidth]{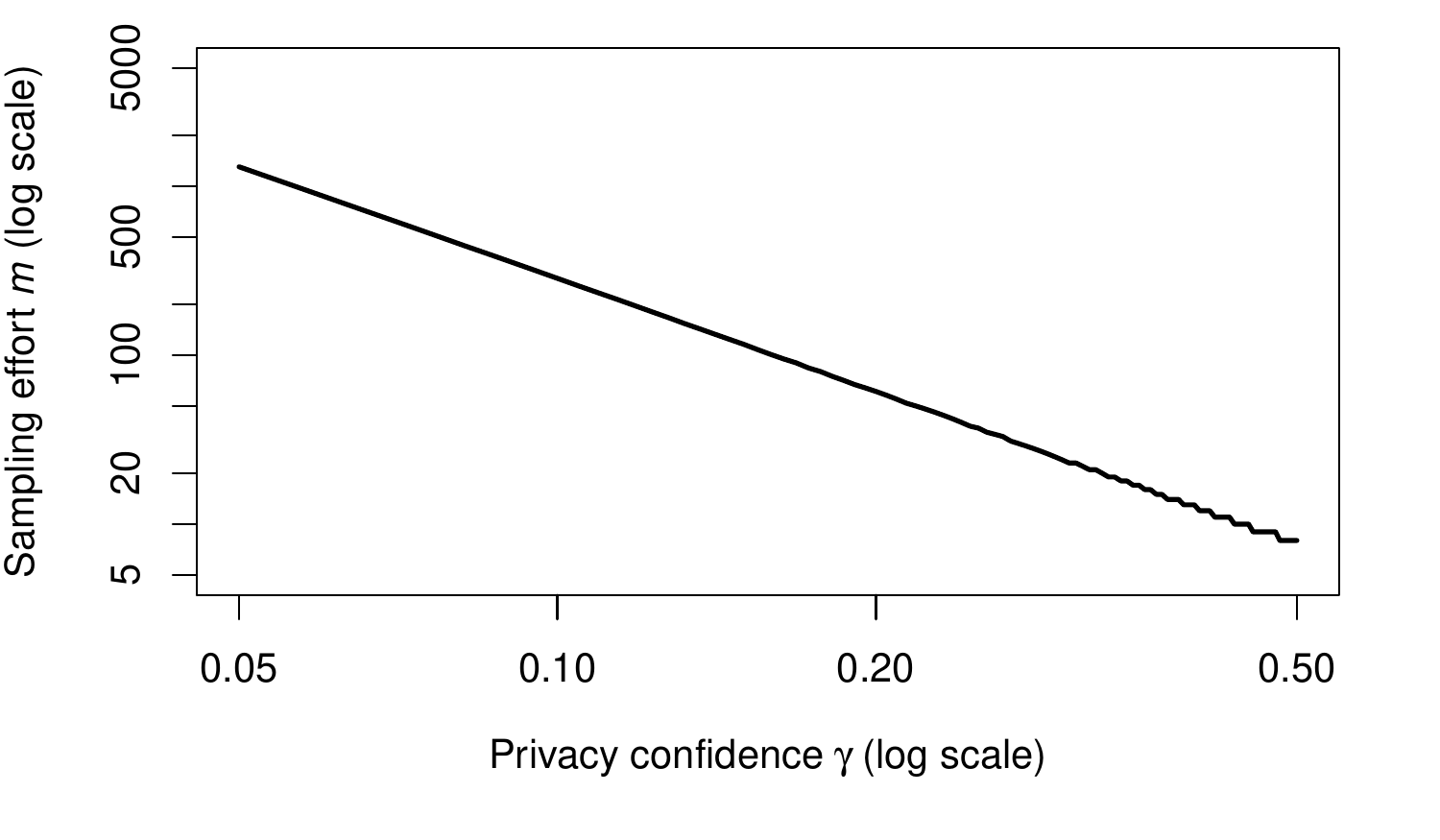}}
\caption{The minimum sample size $m$ (sampler effort) required to achieve various target RDP confidence levels $\gamma$.}
\label{fig:m-vs-gamma}
\end{figure}
\begin{figure}[t]
\centering
\centerline{\includegraphics[width=1.00\columnwidth]{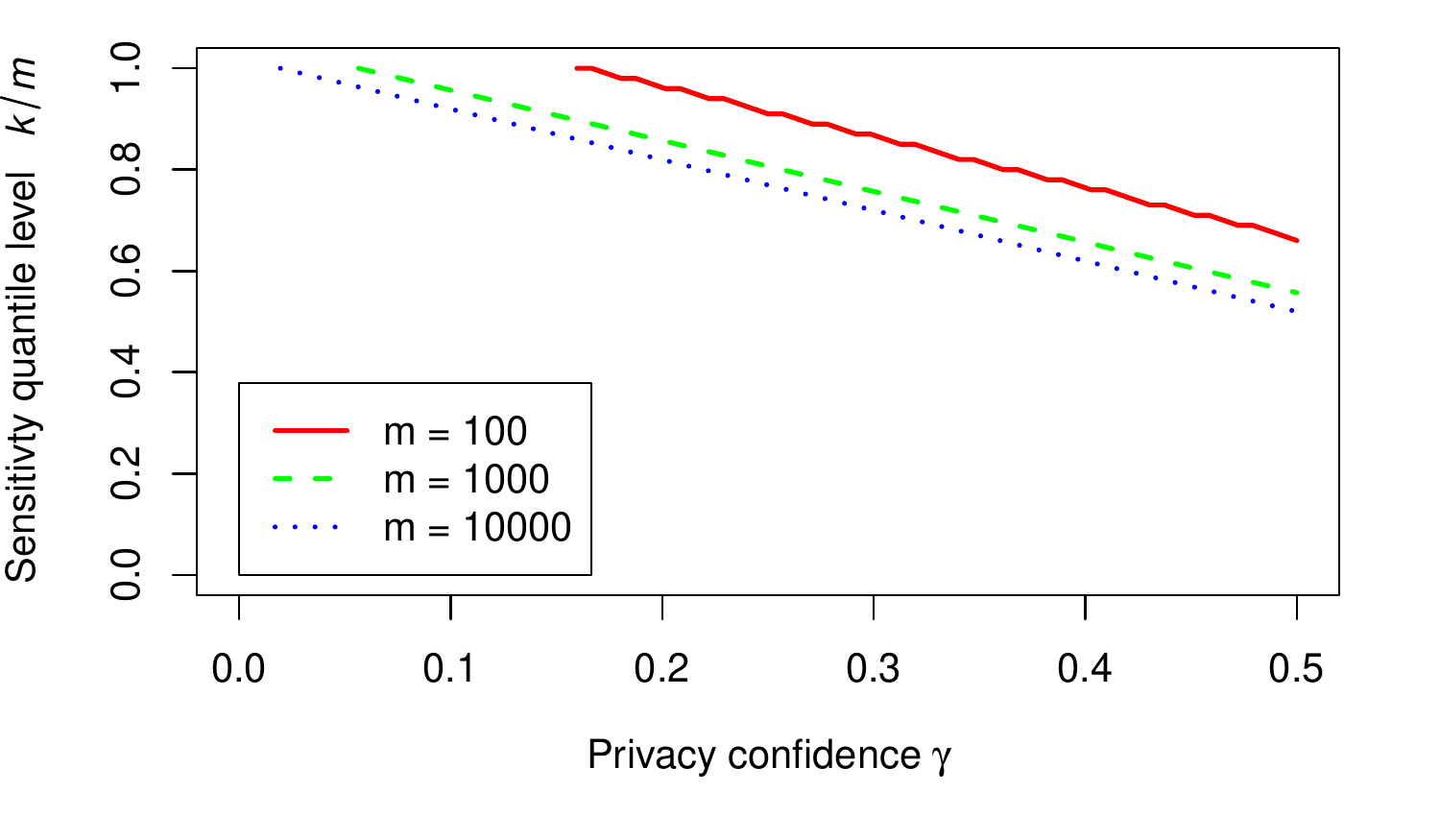}}
\caption{For sample sizes $m\in\{10^2, 10^3, 10^4\}$, trade-offs between privacy confidence level $\gamma$ and order-statistic index $k$ (relative to $m$) which controls sensitivity estimates and so utility.}
\label{fig:k-vs-gamma}
\end{figure}
\begin{align*}
	& \Prs{D,D',G_1,\ldots,G_m}{A_{\hat{\Delta}}} \\
	=& \CExp{\indic{A_{\hat{\Delta}}}}{B_{\rho'}}\Pr{B_{\rho'}} + \CExp{\indic{A_{\hat{\Delta}}}}{\overline{B_{\rho'}}}\Pr{\overline{B_{\rho'}}} \\
	\geq& \CExp{\Phi\left(\hat{\Delta}\right)}{B_{\rho'}} \Pr{B_{\rho'}} \\
	\geq& \CExp{\Phi_m\left(\hat{\Delta}\right) - \rho'}{B_{\rho'}}\left(1-\exp\left(-2m\rho'^2\right)\right) \\
	\geq& \left(1 - \gamma +\rho + \rho' - \rho'\right)\left(1-\exp\left(-2m\rho'^2\right)\right) \\
	\geq& (1 - \gamma + \rho)(1 - \rho) \\
	\geq& 1 - \gamma+\rho - \rho \\
	=& 1 - \gamma \enspace.
\end{align*}
The last inequality follows from $\rho<\gamma$; the penultimate inequality follows from setting
\begin{eqnarray}
	\rho' &\geq& \sqrt{\frac{1}{2m}\log\left(\frac{1}{\rho}\right)}\enspace,\label{eq:rhoprime}
\end{eqnarray}
and so the DKW condition \cite{dkw}, that $\rho'\geq \sqrt{(\log 2)/(2m)}$,
is met provided that $\rho\leq 1/2$. Now~\eqref{thm:main-rdp-condition}
follows from substituting~\eqref{eq:rhoprime} into~\eqref{eq:thm-pf-feasible}.
\end{proof}

Note that for sensitivity-induced $\epsilon$-differentially private mechanisms, the theorem
applies with $\delta=0$.

\myparagraph{Optimising Free Parameter $\boldsymbol{\rho}$.} 
Table~\ref{tab:tuning} recommends alternative choices of free parameter $\rho$, derived by
optimising the sampler's performance along one axis---privacy confidence $\gamma$, sampler
effort $m$, or order statistic index $k$---given a fixed budget of another. 
The table summarises results with proofs found in 
\ifdefined\ARXIV Appendix~\ref{app:optimising}. \else report~\cite{AR2017}. \fi
The specific expressions derived involve branches of the Lambert-$W$ function,
which is the inverse relation of the function $f(z)=z\exp(z)$, and is implemented as a
special function in scientific libraries as standard. While Lambert-$W$ is in general a multi-valued
relation on the analytic complex domain, all instances in our results are single-real-valued
functions on the reals. The next result presents the first operating point's corresponding
rate on effort in terms of privacy, and follows from recent bounds on the secondary branch $W_{-1}$ due to \citet{lambertBound}.\footnote{That for all $u>0$, $-1-\sqrt{2u}-u<W_{-1}(-e^{-u-1}) < -1-\sqrt{2u}-\frac{2}{3}u$.}

\begin{cor}
Minimising $m$ for given $\gamma$ (\cf Table~\ref{tab:tuning}, row 1%
\ifdefined\ARXIV; Corollary~\ref{cor:min-m}, Appendix~\ref{app:optimising}), \else), \fi 
yields rate for 
$m$ as $o\left(\frac{1}{\gamma^2}\log\frac{1}{\gamma}\right)$ with increasing
privacy confidence $\frac{1}{\gamma}\to\infty$.
\end{cor}

\begin{rem}
Theorem~\ref{thm:main-rdp} and Table~\ref{tab:tuning} elucidate that effort, privacy and
utility are in tension. Effort is naturally decreased by reducing the confidence level of RDP
($\rho$ chosen to minimise $m$, or $\gamma$). By minimising order statistic index $k$, we 
select smaller $G_k$ and therefore sensitivity estimate $\hat{\Delta}$. This in turn leads
to lower generic mechanism noise and higher utility. All this is achieved by sacrificing 
effort or privacy confidence. As usual, sacrificing $\epsilon$ or $\delta$ privacy levels
also leads to utility improvement. 
%While Theorem~\ref{thm:main-rdp} appears to suggest that any level of random differential
%privacy can be achieved, with any level of confidence $\rho$---that \sampler provides
%privacy without cost to utility---there are in fact several sources of trade-off.
%First, Condition~\eqref{thm:main-rdp-condition} states that desired
%$(\epsilon,\delta,\gamma)$-random privacy will be achieved w.h.p. provided that either
%$m$ or $\rho$ are sufficiently large (\cf Figure~\ref{fig:m-vs-rho}).
%Second, to achieve superior utility we run our mechanisms with lower sensitivity estimate
%$\hat{\Delta}$ which results from inverting monotonic $\Phi_m(\cdot)$ at a lower level
%$1-\gamma+\rho'$, achieved if either sampling size $m$ is high or sampling confidence is
%low (high $\rho$).
%The effects of existing privacy parameters $\epsilon,\delta,\gamma$ on utility are as usual.
Figures~\ref{fig:m-vs-gamma} and~\ref{fig:k-vs-gamma} visualise these operating points.
\end{rem}

Less conservative estimates on sensitivity can lead to
superior utility while also enjoying easier implementation. This hypothesis is borne out
in experiments in Section~\ref{sec:experiments}.

\begin{prop}
	For any $f:\DBs\to\normed$ with global sensitivity $\overline{\Delta}=\sup_{D\sim D'} \|f(D)-f(D')\|_{\normed}$,
	\sampler's random sensitivity $\hat{\Delta}\leq\overline{\Delta}$. As a result, Algorithm~\ref{algo:mechanism}
	run with any of the sensitivity-induced private mechanisms of Corollaries~\ref{cor:laplace}--\ref{cor:bernstein}
	achieves utility dominating that of the respective mechanisms run with $\overline{\Delta}$.
\end{prop}

\begin{figure}[t]
\centering
\centerline{\includegraphics[width=1.00\columnwidth]{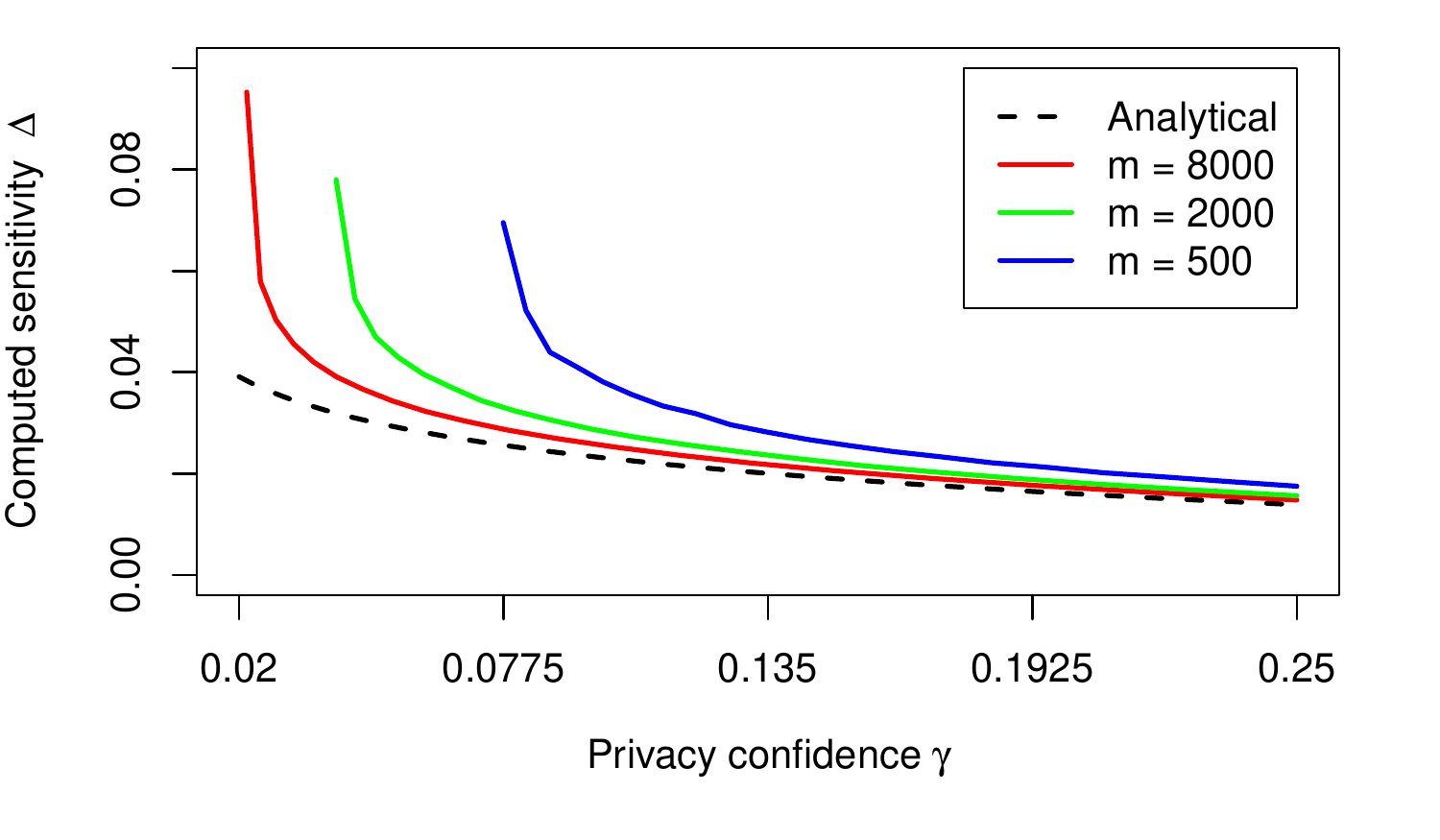}}
\caption{Analytical vs estimated sensitivity for Example~\ref{ex:sample-mean}.}
\label{fig:analytical_vs_sampled}
\end{figure}

\section{Experiments}
\label{sec:experiments}

We now demonstrate the practical
value of \sampler. First in Section~\ref{sec:analytical-vs-sampled} we illustrate
how \sampler sensitivity quickly approaches analytical high-probability sensitivity,
and how it can be significantly lower than worst-case global sensitivity in
Section~\ref{sec:expt-sensitivity}. Running privatising
mechanisms with lower sensitivity parameters can mitigate utility loss,
while maintaining (a weaker form of) differential privacy. We present
experimental evidence of this utility savings in Section~\ref{sec:expt-utility}.
While application domains may find the alternate balance towards utility appealing
by itself, it should be stressed
that a significant advantage of \sampler is its ease of implementation.

\subsection{Analytical RDP vs. Sampled Sensitivity}\label{sec:analytical-vs-sampled}

Consider running Example~\ref{ex:sample-mean}: private release of sample mean
$f(D)=n^{-1}\sum_{i=1}^n D_i$ of a database $D$ drawn i.i.d. from $\mathrm{Exp}(1)$.
Figure~\ref{fig:analytical_vs_sampled} presents, for varying probability $\gamma$:
the analytical bound on sensitivity versus \sampler estimates for different sampling
budgets averaged over 50 repeats. For fixed sampling budget, $\hat{\Delta}$ is
estimated at lower limits on $\gamma$, quickly converging to exact.

\subsection{Global Sensitivity vs. Sampled Sensitivity}\label{sec:expt-sensitivity}

\begin{figure}[t]
\centering
\centerline{\includegraphics[width=1.00\columnwidth]{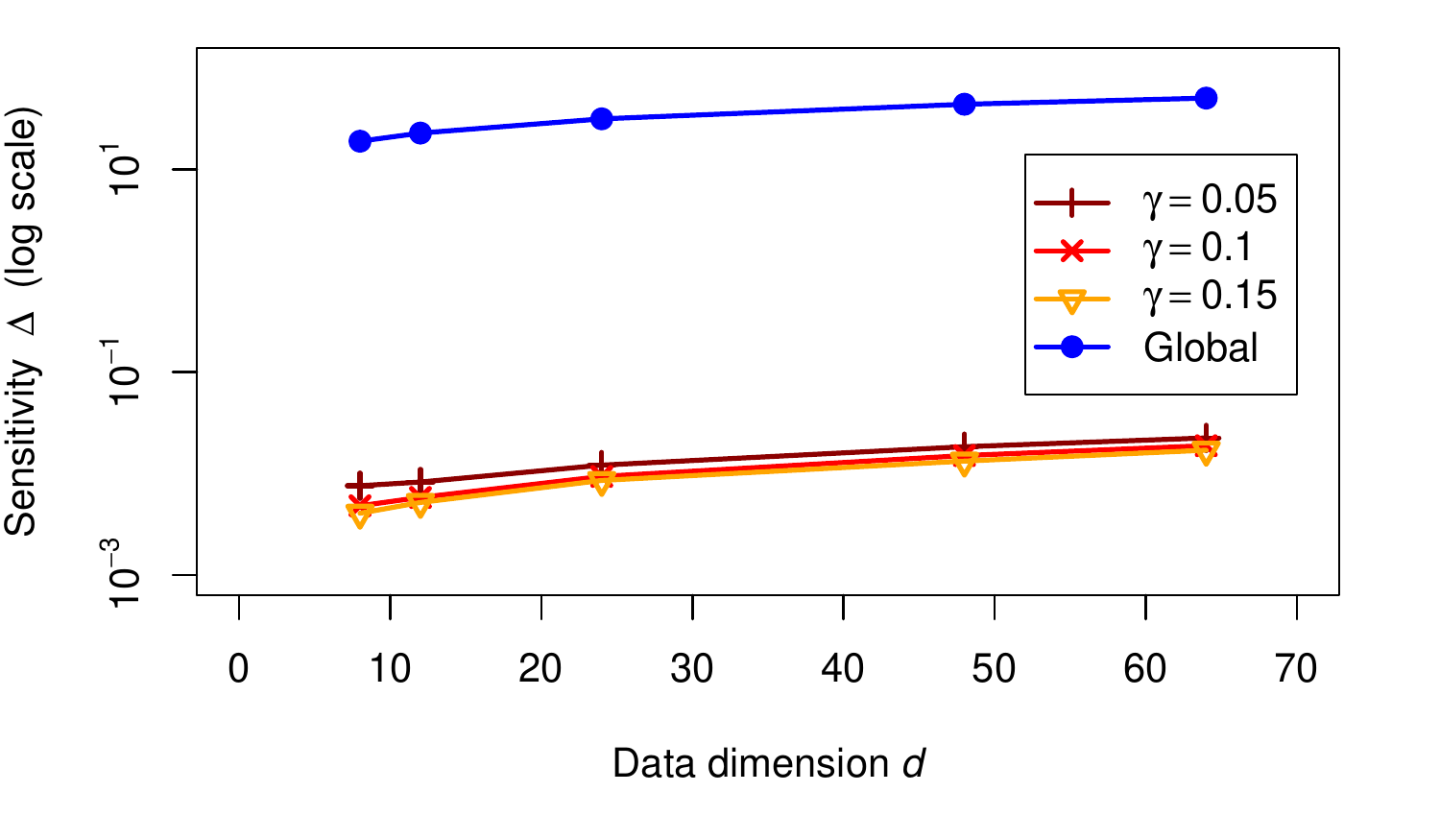}}
\caption{Global vs sampled sensitivity for linear SVM.}
\label{fig:sens_svm_2}
\end{figure}

%\myparagraph{Support Vector Classification.}
Consider now the challenging
goal of privately releasing an SVM classifier fit to sensitive training data.
In applying the Laplace mechanism to releasing the primal normal vector,
\citet{rubinstein2012learning} bound the vector's sensitivity using
algorithmic stability of the SVM. In particular, a lengthy derivation establishes
that $\|\v{w}_{D} - \v{w}_{D'}\|_1\leq 4 L C\kappa\sqrt{d}/n$ for a statistically
consistent formulation of the SVM with convex $L$-Lipschitz loss, $d$-dimensional
feature mapping with $\sup_{\v{x}} k(\v{x},\v{x})\leq\kappa^2$, and regularisation
parameter $C$. While the original work (and others since) did not consider the
practical problem of releasing unregularised bias term $b$, we can effectively
bound this sensitivity via a short argument in 
\ifdefined\ARXIV Appendix~\ref{sec:SVM-proof}. \else full report~\cite{AR2017}. \fi

\begin{figure*}[t!]
	\begin{minipage}[t]{1\columnwidth}
\centering
\centerline{\includegraphics[width=1.00\columnwidth]{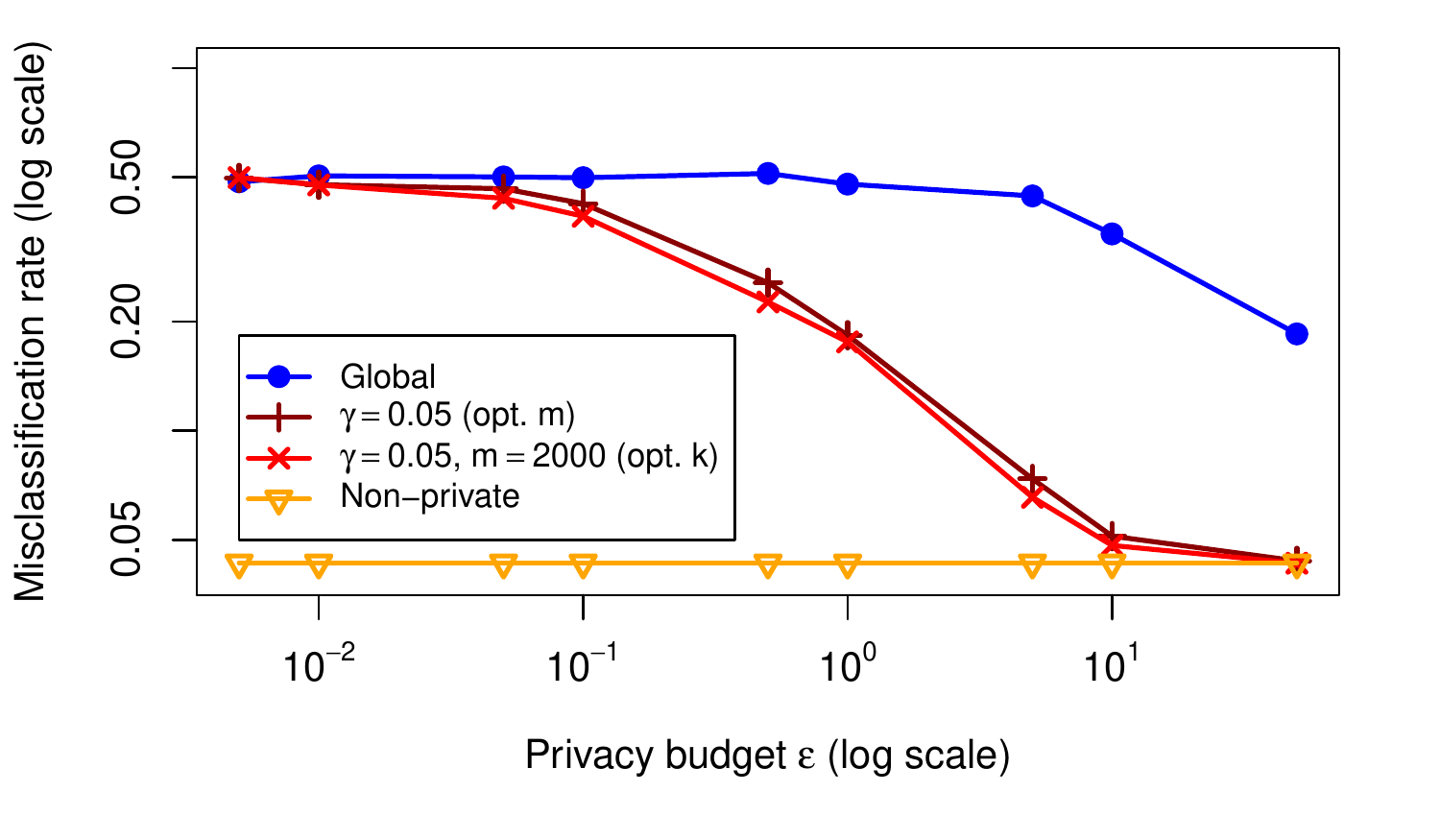}}
\caption{Linear SVM predictive error under sensitivity estimates vs with global sensitivity bound.}
\label{fig:utility_svm_2}
\end{minipage}\hfill
\begin{minipage}[t]{1\columnwidth}
\centering
\centerline{\includegraphics[width=1.00\columnwidth]{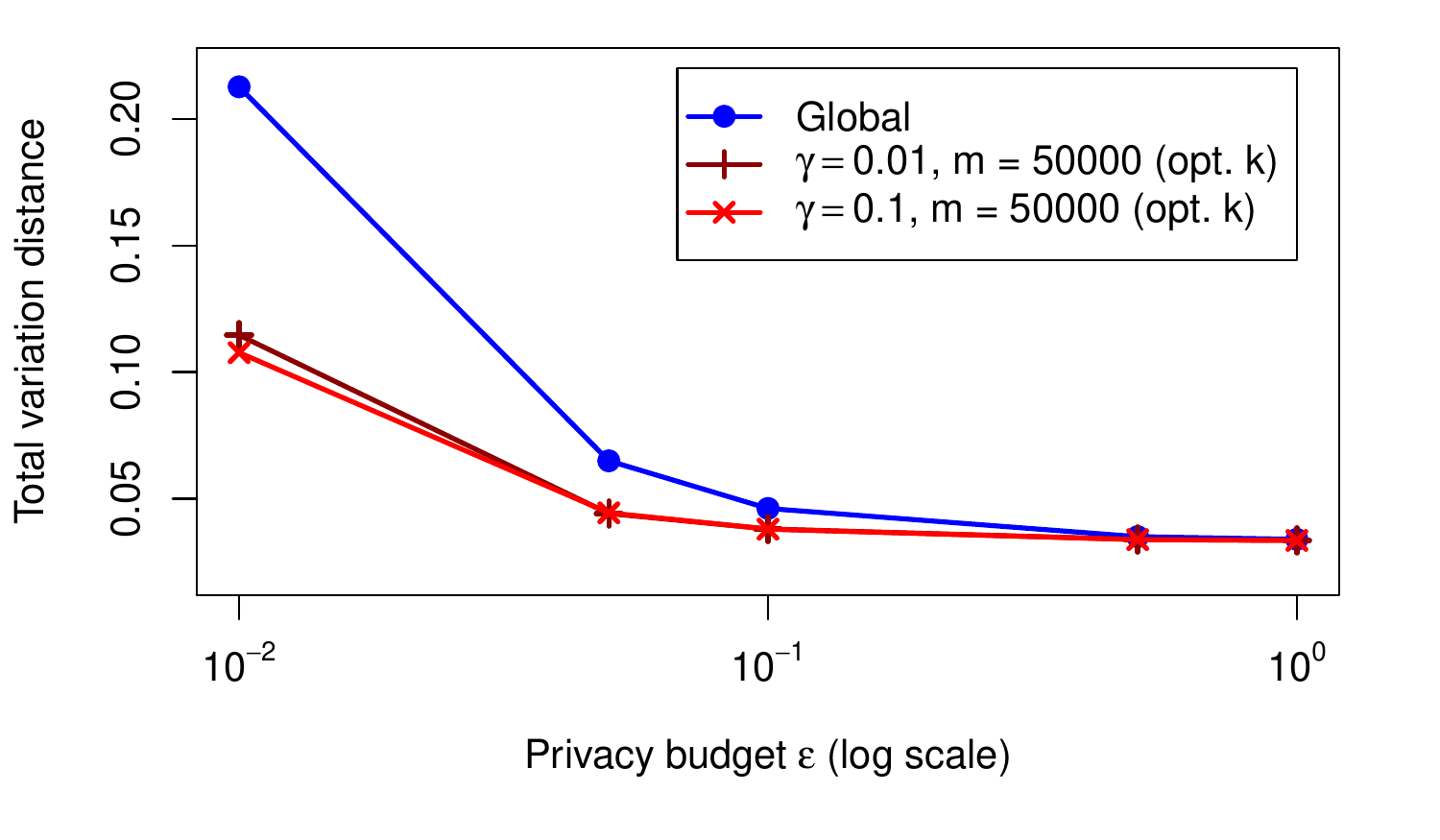}}
\caption{KDE error (relative to non-private) under sensitivity estimates vs global sensitivity bound.}
\label{fig:utility_kde}
\end{minipage}
\end{figure*}

\begin{prop}\label{prop:SVM}
For the SVM run with hinge loss, linear kernel, $\domain=[0,1]^d$, the release
$(\v{w},b)$ has $L_1$ global sensitivity bounded by $2+2 C \sqrt{d} + 4 C d / n$.
\end{prop}

We train private SVM using the Laplace mechanism~\cite{rubinstein2012learning},
with global sensitivity bound of Proposition~\ref{prop:SVM} or \sampler.
We synthesise a dataset of $n=1000$ points, selected with equal
probability of being drawn from the positive class $N(0.2\cdot\v{1}, \mathrm{diag}(0.01))$
or negative class $N(0.8\cdot\v{1}, \mathrm{diag}(0.01))$. The feature space's
dimension varies from $d=8$ through $d=64$. The SVMs are run
with $C=3$, \sampler with $m=1500$ \& varying $\gamma$. 
Figure~\ref{fig:sens_svm_2} shows very different sensitivities obtained. While
estimated $\hat{\Delta}$ hovers around 0.01 largely independent of $\gamma$, 
global sensitivity $\overline{\Delta}$ exceeds 20---two orders of magnitude
greater. These patterns are repeated as dimension increases; sensitivity increasing
is to be expected since as dimensions are added, the few points in the training set
become more likely to be support vectors and thus affecting sensitivity.
Such conservative estimates could clearly lead to inferior utility.

\subsection{Effect on Utility}\label{sec:expt-utility}

\myparagraph{Support Vector Classification.}
We return to the same SVM setup as in the previous section, with $d=2$, now plotting
utility as misclassification error (averaged over 500 repeats) vs. privacy budget $\epsilon$.
Here we set $\gamma=0.05$ and include also the non-private SVM's performance as a bound on
utility possible. See Figure~\ref{fig:utility_svm_2}. At very high privacy levels both
private SVMs suffer the same poor error. But quickly with lower privacy, the
misclassification error of \sampler drops until it reaches the non-private rate.
Simultaneously the global sensitivity approach has a significantly higher value and
suffers a much slower decline. These results suggest that \sampler
can achieve much better utility in addition to sensitivity.

\myparagraph{Kernel Density Estimation.}
We finally consider a one dimensional ($d=1$) KDE setting. In Figure~\ref{fig:utility_kde} we show the error (averaged over 1000 repeats) of the Bernstein mechanism (with lattice size $k = 10$ and Bernstein order $h=3$) on 5000 points drawn from a mixture of two normal distributions $N(0.5, 0.02)$ and $N(0.75, 0.005)$ with weights $0.4$, $0.6$, respectively. For this experimental result, we set $m=50000$ and two different values for $\gamma$, as displayed in Figure~\ref{fig:utility_kde}. Once again we observe that for high privacy levels the global sensitivity approach incurs a higher error relative to non-private, while \sampler provides stronger utility. At lower privacy, both approaches converge to the approximation error of the Bernstein polynomial used.

\section{Conclusion}
\label{sec:conclusion}

In this paper we propose \sampler, an algorithm for empirical estimation of sensitivity
for privatisation of black-box functions. Our work addresses an
important usability gap in differential privacy, whereby several generic privatisation
mechanisms exist complete with privacy and utility guarantees, but require analytical bounds
on global sensitivity (a Lipschitz condition) on the non-private target. While this sensitivity
is trivially derived for simple statistics, for state-of-the-art learners 
sensitivity derivations are arduous \eg in collaborative
filtering~\cite{mcsherry2009differentially},
SVMs~\cite{rubinstein2012learning,chaudhuri2011differentially},
model selection~\cite{thakurta2013differentially}, feature selection~\cite{kifer2012private},
Bayesian inference~\cite{dimitrakakis2014robust,wang2015privacy}, and deep
learning~\cite{abadi2016deep}.

While derivations may prevent domain experts from
leveraging differential privacy, our \sampler promises to make privatisation simple when
 using existing mechanisms including Laplace~\cite{dwork2006calibrating},
Gaussian~\cite{dwork2014algorithmic}, exponential~\cite{mcsherry2007mechanism}
and Bernstein~\cite{alda2017bernstein}. All such mechanisms guarantee
differential privacy on pairs of databases for which a level $\Delta$ of non-private function
sensitivity holds, when the mechanism is run with that $\Delta$ parameter. For all such mechanisms
we leverage results from empirical process theory to establish guarantees
of random differential privacy~\cite{hall2012random} when using sampled sensitivities only.

Experiments
demonstrate that real-world learners can easily be run privately without any new derivation whatsoever.
And by using a naturally-weaker form of privacy, while replacing worst-case global sensitivity bounds with
estimated (actual) sensitivities, we can achieve far superior utility than existing approaches.

%Citations within the text should include the authors' last names and year. If the authors' names are included in the sentence, place only the year in parentheses, for example when referencing Arthur Samuel's pioneering work \yrcite{Samuel59}. Otherwise place the entire reference in parentheses with the authors and year separated by a comma \cite{Samuel59}. List multiple references separated by semicolons \cite{kearns89,Samuel59,mitchell80}. Use the `et~al.' construct only for citations with three or more authors or after listing all authors to a publication in an earlier reference \cite{MachineLearningI}.

% Acknowledgements should only appear in the accepted version.
\section*{Acknowledgements}

F. Ald\`a and B. Rubinstein acknowledge the support of the DFG Research Training Group GRK 1817/1 and the Australian Research Council (DE160100584) respectively.

\bibliography{refs}

\begin{thebibliography}{30}
\providecommand{\natexlab}[1]{#1}
\providecommand{\url}[1]{\texttt{#1}}
\expandafter\ifx\csname urlstyle\endcsname\relax
  \providecommand{\doi}[1]{doi: #1}\else
  \providecommand{\doi}{doi: \begingroup \urlstyle{rm}\Url}\fi

\bibitem[Abadi et~al.(2016)Abadi, Chu, Goodfellow, McMahan, Mironov, Talwar,
  and Zhang]{abadi2016deep}
Abadi, Mart{\'\i}n, Chu, Andy, Goodfellow, Ian, McMahan, H~Brendan, Mironov,
  Ilya, Talwar, Kunal, and Zhang, Li.
\newblock Deep learning with differential privacy.
\newblock In \emph{Proceedings of the 2016 ACM SIGSAC Conference on Computer
  and Communications Security}, pp.\  308--318. ACM, 2016.

\bibitem[Ald{\`a} \& Rubinstein(2017)Ald{\`a} and
  Rubinstein]{alda2017bernstein}
Ald{\`a}, Francesco and Rubinstein, Benjamin I.~P.
\newblock The {B}ernstein mechanism: Function release under differential
  privacy.
\newblock In \emph{Proceedings of the 31st AAAI Conference on Artificial
  Intelligence (AAAI'2017)}, pp.\  1705--1711, 2017.

\bibitem[Barthe et~al.(2016)Barthe, Gaboardi, Hsu, and
  Pierce]{barthe2016programming}
Barthe, Gilles, Gaboardi, Marco, Hsu, Justin, and Pierce, Benjamin.
\newblock Programming language techniques for differential privacy.
\newblock \emph{ACM SIGLOG News}, 3\penalty0 (1):\penalty0 34--53, 2016.

\bibitem[Chatzigeorgiou(2013)]{lambertBound}
Chatzigeorgiou, Ioannis.
\newblock Bounds on the {L}ambert function and their application to the outage
  analysis of user cooperation.
\newblock \emph{IEEE Communications Letters}, 17\penalty0 (8), 2013.

\bibitem[Chaudhuri et~al.(2011)Chaudhuri, Monteleoni, and
  Sarwate]{chaudhuri2011differentially}
Chaudhuri, Kamalika, Monteleoni, Claire, and Sarwate, Anand~D.
\newblock Differentially private empirical risk minimization.
\newblock \emph{Journal of Machine Learning Research}, 12\penalty0
  (Mar):\penalty0 1069--1109, 2011.

\bibitem[Dimitrakakis et~al.(2014)Dimitrakakis, Nelson, Mitrokotsa, and
  Rubinstein]{dimitrakakis2014robust}
Dimitrakakis, Christos, Nelson, Blaine, Mitrokotsa, Aikaterini, and Rubinstein,
  Benjamin I.~P.
\newblock Robust and private {B}ayesian inference.
\newblock In \emph{International Conference on Algorithmic Learning Theory},
  pp.\  291--305. Springer, 2014.

\bibitem[Dimitrakakis et~al.(2017)Dimitrakakis, Nelson, Zhang, Mitrokotsa, and
  Rubinstein]{dimitrakakis2014robustNew}
Dimitrakakis, Christos, Nelson, Blaine, Zhang, Zuhe, Mitrokotsa, Aikaterini,
  and Rubinstein, Benjamin I.~P.
\newblock Differential privacy for {B}ayesian inference through posterior
  sampling.
\newblock \emph{Journal of Machine Learning Research}, 18\penalty0
  (11):\penalty0 1--39, 2017.

\bibitem[Dwork \& Roth(2014)Dwork and Roth]{dwork2014algorithmic}
Dwork, Cynthia and Roth, Aaron.
\newblock The algorithmic foundations of differential privacy.
\newblock \emph{Foundations and Trends in Theoretical Computer Science},
  9\penalty0 (3--4):\penalty0 211--407, 2014.

\bibitem[Dwork et~al.(2006)Dwork, McSherry, Nissim, and
  Smith]{dwork2006calibrating}
Dwork, Cynthia, McSherry, Frank, Nissim, Kobbi, and Smith, Adam.
\newblock Calibrating noise to sensitivity in private data analysis.
\newblock In \emph{Theory of Cryptography Conference}, pp.\  265--284.
  Springer, 2006.

\bibitem[Gaboardi et~al.(2013)Gaboardi, Haeberlen, Hsu, Narayan, and
  Pierce]{gaboardi2013linear}
Gaboardi, Marco, Haeberlen, Andreas, Hsu, Justin, Narayan, Arjun, and Pierce,
  Benjamin~C.
\newblock Linear dependent types for differential privacy.
\newblock \emph{ACM SIGPLAN Notices}, 48\penalty0 (1):\penalty0 357--370, 2013.

\bibitem[Haeberlen et~al.(2011)Haeberlen, Pierce, and
  Narayan]{haeberlen2011differential}
Haeberlen, Andreas, Pierce, Benjamin~C, and Narayan, Arjun.
\newblock Differential privacy under fire.
\newblock In \emph{USENIX Security Symposium}, 2011.

\bibitem[Hall et~al.(2012)Hall, Rinaldo, and Wasserman]{hall2012random}
Hall, Rob, Rinaldo, Alessandro, and Wasserman, Larry.
\newblock Random differential privacy.
\newblock \emph{Journal of Privacy and Confidentiality}, 4\penalty0
  (2):\penalty0 43--59, 2012.

\bibitem[Kifer et~al.(2012)Kifer, Smith, and Thakurta]{kifer2012private}
Kifer, Daniel, Smith, Adam, and Thakurta, Abhradeep.
\newblock Private convex empirical risk minimization and high-dimensional
  regression.
\newblock \emph{Journal of Machine Learning Research}, 1\penalty0
  (41):\penalty0 3--1, 2012.

\bibitem[Massart(1990)]{dkw}
Massart, Pascal.
\newblock The tight constant in the {D}voretzky-{K}iefer-{W}olfowitz
  inequality.
\newblock \emph{The Annals of Probability}, 18\penalty0 (3):\penalty0
  1269--1283, 1990.

\bibitem[McSherry \& Mahajan(2010)McSherry and
  Mahajan]{mcsherry2010differentially}
McSherry, Frank and Mahajan, Ratul.
\newblock Differentially-private network trace analysis.
\newblock \emph{ACM SIGCOMM Computer Communication Review}, 40\penalty0
  (4):\penalty0 123--134, 2010.

\bibitem[McSherry \& Mironov(2009)McSherry and
  Mironov]{mcsherry2009differentially}
McSherry, Frank and Mironov, Ilya.
\newblock Differentially private recommender systems: building privacy into the
  net.
\newblock In \emph{Proceedings of the 15th ACM SIGKDD International Conference
  on Knowledge Discovery and Data Mining}, pp.\  627--636. ACM, 2009.

\bibitem[McSherry \& Talwar(2007)McSherry and Talwar]{mcsherry2007mechanism}
McSherry, Frank and Talwar, Kunal.
\newblock Mechanism design via differential privacy.
\newblock In \emph{48th Annual IEEE Symposium on Foundations of Computer
  Science, 2007 (FOCS'07)}, pp.\  94--103. IEEE, 2007.

\bibitem[McSherry(2009)]{mcsherry2009privacy}
McSherry, Frank~D.
\newblock Privacy integrated queries: an extensible platform for
  privacy-preserving data analysis.
\newblock In \emph{Proceedings of the 2009 ACM SIGMOD International Conference
  on Management of Data}, pp.\  19--30. ACM, 2009.

\bibitem[Minami et~al.(2016)Minami, Arai, Sato, and Nakagawa]{NIPS2016_6050}
Minami, Kentaro, Arai, HItomi, Sato, Issei, and Nakagawa, Hiroshi.
\newblock Differential privacy without sensitivity.
\newblock In \emph{Advances in Neural Information Processing Systems 29}, pp.\
  956--964, 2016.

\bibitem[Mir(2012)]{mir2012differentially}
Mir, Darakhshan.
\newblock Differentially-private learning and information theory.
\newblock In \emph{Proceedings of the 2012 Joint EDBT/ICDT Workshops}, pp.\
  206--210. ACM, 2012.

\bibitem[Mohan et~al.(2012)Mohan, Thakurta, Shi, Song, and
  Culler]{mohan2012gupt}
Mohan, Prashanth, Thakurta, Abhradeep, Shi, Elaine, Song, Dawn, and Culler,
  David.
\newblock {GUPT}: privacy preserving data analysis made easy.
\newblock In \emph{Proceedings of the 2012 ACM SIGMOD International Conference
  on Management of Data}, pp.\  349--360. ACM, 2012.

\bibitem[Nissim et~al.(2007)Nissim, Raskhodnikova, and Smith]{nissim2007smooth}
Nissim, Kobbi, Raskhodnikova, Sofya, and Smith, Adam.
\newblock Smooth sensitivity and sampling in private data analysis.
\newblock In \emph{Proceedings of the Thirty-Ninth Annual ACM Symposium on
  Theory of Computing}, pp.\  75--84. ACM, 2007.

\bibitem[Palamidessi \& Stronati(2012)Palamidessi and Stronati]{palamidessi}
Palamidessi, Catuscia and Stronati, Marco.
\newblock {Differential privacy for relational algebra: improving the
  sensitivity bounds via constraint systems}.
\newblock In Wiklicky, Herbert and Massink, Mieke (eds.), \emph{{QAPL - Tenth
  Workshop on Quantitative Aspects of Programming Languages}}, volume~85, pp.\
  92--105, 2012.

\bibitem[Reed \& Pierce(2010)Reed and Pierce]{reed2010distance}
Reed, Jason and Pierce, Benjamin~C.
\newblock Distance makes the types grow stronger: a calculus for differential
  privacy.
\newblock \emph{ACM Sigplan Notices}, 45\penalty0 (9):\penalty0 157--168, 2010.

\bibitem[Riondato \& Upfal(2015)Riondato and Upfal]{riondato2015mining}
Riondato, Matteo and Upfal, Eli.
\newblock Mining frequent itemsets through progressive sampling with
  {R}ademacher averages.
\newblock In \emph{Proceedings of the 21th ACM SIGKDD International Conference
  on Knowledge Discovery and Data Mining}, pp.\  1005--1014. ACM, 2015.

\bibitem[Roy et~al.(2010)Roy, Setty, Kilzer, Shmatikov, and
  Witchel]{roy2010airavat}
Roy, Indrajit, Setty, Srinath~TV, Kilzer, Ann, Shmatikov, Vitaly, and Witchel,
  Emmett.
\newblock Airavat: Security and privacy for {M}ap{R}educe.
\newblock In \emph{NSDI}, volume~10, pp.\  297--312, 2010.

\bibitem[Rubinstein et~al.(2012)Rubinstein, Bartlett, Huang, and
  Taft]{rubinstein2012learning}
Rubinstein, Benjamin I.~P., Bartlett, Peter~L., Huang, Ling, and Taft, Nina.
\newblock Learning in a large function space: Privacy-preserving mechanisms for
  {SVM} learning.
\newblock \emph{Journal of Privacy and Confidentiality}, 4\penalty0
  (1):\penalty0 65--100, 2012.

\bibitem[Thakurta \& Smith(2013)Thakurta and Smith]{thakurta2013differentially}
Thakurta, Abhradeep~Guha and Smith, Adam.
\newblock Differentially private feature selection via stability arguments, and
  the robustness of the {L}asso.
\newblock In \emph{Conference on Learning Theory}, pp.\  819--850, 2013.

\bibitem[Wang et~al.(2015)Wang, Fienberg, and Smola]{wang2015privacy}
Wang, Yu-Xiang, Fienberg, Stephen~E, and Smola, Alexander~J.
\newblock Privacy for free: Posterior sampling and stochastic gradient {M}onte
  {C}arlo.
\newblock In \emph{ICML}, pp.\  2493--2502, 2015.

\bibitem[Zhang et~al.(2016)Zhang, Rubinstein, and
  Dimitrakakis]{zhang2016differential}
Zhang, Zuhe, Rubinstein, Benjamin I.~P., and Dimitrakakis, Christos.
\newblock On the differential privacy of {B}ayesian inference.
\newblock In \emph{Proceedings of the Thirtieth AAAI Conference on Artificial
  Intelligence}, pp.\  2365--2371. AAAI Press, 2016.

\end{thebibliography}
\bibliographystyle{icml2017}

\ifdefined\ARXIV

\appendix

\section{Proof of Proposition~\ref{prop:transferRDP}}
\label{sec:transferrdp-proof}

By Pinsker's inequality the product measures have bounded total variation distance
$$\left\| P^{n+1} - Q^{n+1}\right\|\leq\sqrt{\frac{1}{2}KL\left(P^{n+1}\| Q^{n+1}\right)}\leq\sqrt{\frac{n+1}{2}\tau}.$$

Denote by $A$ the event that $\epsilon$-DP holds (similarly for $(\epsilon,\delta)$-DP) on neighbouring
databases on $n$ records:
$$A = \left\{\forall R\subset\responses, \Pr{M(D)\in R}\leq e^{\epsilon}\Pr{M(D')\in R}\right\}.$$

Then RDP wrt $Q$ follows as
\begin{eqnarray*}
Q^{n+1}(A) &\geq& P^{n+1}(A) - \sqrt{(n+1)\tau/2} \\
&\geq& 1 - \gamma - \sqrt{(n+1)\tau/2}\enspace.
\end{eqnarray*}

\section{Optimising Sampler Performance with $\boldsymbol{\rho}$}
\label{app:optimising}

This section presents precise statements and proofs for the expressions found in Table~\ref{tab:tuning}.

\subsection{Fixed $\boldsymbol{\gamma}$ Minimum $\boldsymbol{m}$}

\begin{cor}\label{cor:min-m}
For fixed given privacy confidence budget $\gamma\in(0,1)$, taking 
\begin{eqnarray*}
\rho &=& \exp\left(W_{-1}\left(-\frac{\gamma}{2\sqrt{e}}\right)+\frac{1}{2}\right)\enspace, \nonumber \\
		m &=& \left\lceil\left(2(\gamma-\rho)^2\right)^{-1} \log(1/\rho)\right\rceil\enspace, \\ % \label{eq:setting-m} \\
	k &=& \left\lceil m\left( 1 -\gamma +\rho + \sqrt{\log(1/\rho) / (2m)} \right)\right\rceil \enspace, %\label{eq:setting-k} 
\end{eqnarray*}
minimises sampling effort $m$, when running Algorithm~\ref{algo:mechanism} to achieve 
$(\epsilon,\delta,\gamma)$-RDP.
%Under this regime, sample size $m$ grows as 
%$o\left(\frac{1}{\gamma^2}\log\frac{1}{\gamma}\right)$ with increasing confidence
%$\frac{1}{\gamma}\to\infty$.
\end{cor}

%\begin{proof}
%We have demonstrated a collection of bounds on $m$ indexed by $\rho\in(0,0.5]$.
%Appendix~\ref{app:optimising} proves that these bounds are strictly convex in
%$\rho$ and are minimised by taking
%$\rho=\exp\left(W_{-1}\left(-\frac{\gamma}{2\sqrt{e}}\right)+\frac{1}{2}\right)$
%where $W_{-1}$, the lower-branch of the Lambert-W function, is real-valued
%over feasible $\gamma$. The rate for the bound derives from bounds on
%Lambert-$W$ due to \citet{lambertBound}.
%\end{proof}

\begin{proof}
For any fixed $\gamma\in(0,1)$, our task is to minimise the bound
\begin{eqnarray*}
m(\rho) &=& \frac {1}{2(\gamma-\rho)^2}\log\frac{1}{\rho}\enspace,
\end{eqnarray*}
on $\rho\in(0,\min\{\gamma,0.5\})$. The first- and second-order derivatives of this
function are
\begin{align*}
	\frac{\partial m}{\partial\rho} &= -\frac{\log\rho}{(\gamma-\rho)^3}-\frac{1}{2\rho(\gamma-\rho)^2} \\
	\frac{\partial^2 m}{\partial\rho^2} &= -\frac{3\log\rho}{(\gamma-\rho)^4}-\frac{2}{\rho(\gamma-\rho)^3} + \frac{1}{2\rho^2(\gamma - \rho)^2} \\
	&= \frac{1}{2\rho^2(\gamma-\rho)^4} \left[ (\gamma-3\rho)^2 + \rho^2\left(6\log\frac{1}{\rho} - 4\right)\right].
\end{align*}
For the second derivative to be positive, it is sufficient for $6\log\frac{1}{\rho} - 4>0$ 
which in turn is guaranteed when $\rho<\exp(-2/3)\approx 0.51$. Therefore $m(\rho)$ is strictly convex on the feasible region; and the first-order necessary
condition for optimality is also sufficient. We seek $\rho^\star$ critical point
\begin{eqnarray*}
	0 &=& -(\gamma-\rho^\star)^{-2}\left[\frac{\log\rho^\star}{\gamma-\rho^\star} + \frac{1}{2\rho^\star}\right] \\
	\Leftrightarrow -\frac{\log\rho^\star}{\gamma-\rho^\star} &=& \frac{1}{2\rho^\star} \\
	\Leftrightarrow -\gamma &=& 2\rho^\star \log\rho^\star - \rho^\star \\
	\Leftrightarrow -\frac{\gamma}{2\sqrt{e}} &=& \left(\log\rho^\star - \frac{1}{2}\right)\frac{\rho^\star}{\sqrt{e}} \\
&=& \left(\log\rho^\star - \frac{1}{2}\right)\exp\left(\log\rho^\star - \frac{1}{2}\right)
\end{eqnarray*}
Applying the Lambert-$W$ function to each side, yields
\begin{eqnarray*}
	\log(\rho^\star) - \frac{1}{2} &\in& W\left(-\frac{\gamma}{2\sqrt{e}}\right) \\
	\Leftrightarrow \rho^\star &\in& \exp\left(W\left(-\frac{\gamma}{2\sqrt{e}}\right)+\frac{1}{2}\right)\enspace.
\end{eqnarray*}

\begin{figure}[t!]
\begin{center}
\centerline{\includegraphics[width=1.00\columnwidth]{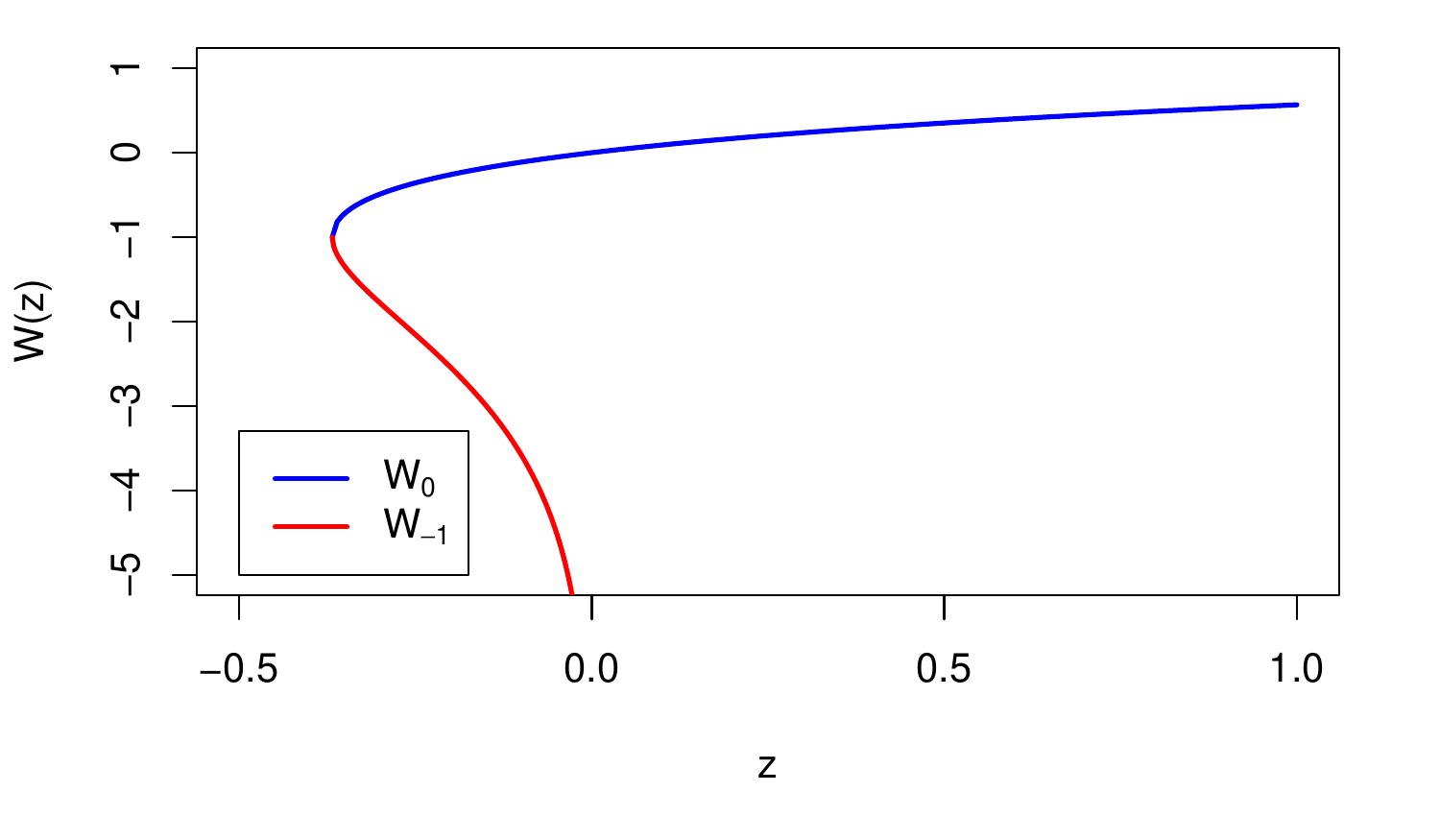}}
\caption{The branches of the Lambert-W function: primary $W_0$ (blue) and secondary $W_{-1}$ (red).}
\label{fig:lambert}
\end{center}
\end{figure}

The Lambert-$W$ function is real-valued on $[-\exp(-1),\infty)$, within which
it is two-valued on $(-\exp(-1),0)$ and univalued otherwise. As depicted in
Figure~\ref{fig:lambert}, it consists of a primary branch $W_0$ which maps
$[-\exp(-1),\infty)$ to $[-1,\infty)$, and a secondary branch $W_{-1}$ which
maps $[-\exp(-1),0)$ to $[-1,-\infty)$. Returning to our condition on
$\rho^\star$, consider that for $\gamma\in(0,1)$ we have that
$-\frac{\gamma}{2\sqrt{e}}$ since $2>\sqrt{e}$. On this domain primary
$W_0\in(-1,0)$ while secondary $W_{-1}\in(-\infty,-1)$ and so the primary
branch would yield $\rho^\star\in(-\exp(-0.5), \exp(0.5))$ which is disjoint
from feasible region $(0,0.5]$. The secondary branch, however, has image
in $(0, 0.28457\ldots)$ which is feasible. Therefore, we arrive at the 
$\rho^\star$ as claimed, completing
the main part of the proof.
%The rate for $m(\rho^\star)$ follows from bounds on the secondary branch
%$W_{-1}$ due to \citet{lambertBound}:
%
%\begin{thm}
%	For $u>0$, the secondary branch is bounded as
%	$-1-\sqrt{2u}-u<W_{-1}(-\exp(-u-1)) < -1-\sqrt{2u}-\frac{2}{3}u$.
%\end{thm}
\end{proof}

\subsection{Fixed $\boldsymbol{m}$ and $\boldsymbol{\gamma}$ Minimum $\boldsymbol{k}$}
%\label{app:optimising-k}

\begin{cor}\label{cor:min-k}
	For given fixed sampling resource budget $m\in\naturals$ and privacy confidence $\gamma\in(0,1)$, taking
\begin{eqnarray*}
\rho &=& \exp\left(\frac{1}{2}W_{-1}\left(-\frac{1}{4m}\right)\right)\enspace, \\ %\label{eq:opt-rho-k} 
k &=& \left\lceil m\left( 1 -\gamma +\rho + \sqrt{\log(1/\rho) / (2m)} \right)\right\rceil \enspace, %\label{eq:setting-k} 
\end{eqnarray*}
provided that
\begin{eqnarray*}
\gamma &\geq& \rho + \sqrt{\log(1/\rho)/(2m)}\enspace, %\label{eq:opt-gamma-k}
\end{eqnarray*}
minimises order-statistic index $k$, when running Algorithm~\ref{algo:mechanism} to achieve 
$(\epsilon,\delta,\gamma)$-RDP.
\end{cor}

\begin{proof}
For fixed $m, \gamma$, our task is to minimise the bound
\begin{eqnarray*}
k(\rho) &=& m\left(1 - \gamma + \rho + \sqrt{\frac{\log(1/\rho)}{2m}}\right)\enspace,
\end{eqnarray*}
or equivalently
\begin{eqnarray}
	\tilde{k}(\rho) &=& \rho + \sqrt{\frac{\log(1/\rho)}{2m}}\enspace, \label{eq:tilde-k}
\end{eqnarray}
on $\rho\in(0, \min\{\gamma,0.5\})$. The first- and second-order derivatives of this
function are
\begin{align*}
	\frac{\partial \tilde{k}}{\partial\rho} &= 1 - \frac{1}{2\sqrt{2m}\rho\sqrt{\log(1/\rho)}} \\
	\frac{\partial^2 \tilde{k}}{\partial\rho^2} &= \frac{1}{2\sqrt{2m}\rho^2\sqrt{\log(1/\rho)}}\left[1 - \frac{1}{2\log(1/\rho)}\right]\enspace.
\end{align*}
Since its leading term is positive on feasible $\rho$, it
follows that the second derivative is strictly positive iff $\rho<\exp(-1/2)\approx 0.6$ 
which is guaranteed on the feasible region. Therefore $k(\rho)$ is strictly convex; and the
first-order necessary condition for optimality is also sufficient.
Next we seek $\rho^\star$ critical point
\begin{eqnarray*}
	0 &=& 1 - \frac{1}{2\sqrt{2m}\rho^\star\sqrt{\log(1/\rho^\star)}} \\
	\Leftrightarrow \rho^{\star 2}\log\rho^\star &=& - \frac{1}{8m} \\
	\Leftrightarrow \rho^{\star ^2} \log\rho^{\star ^2} &=& - \frac{1}{4m} \\
	\Leftrightarrow \log \rho^{\star ^2} &=& W\left(-\frac{1}{4m}\right) \\
	\Leftrightarrow \rho^\star &\in& \exp\left(\frac{1}{2} W\left(-\frac{1}{4m}\right)\right)\enspace,
\end{eqnarray*}
where the introduction of the Lambert-$W$ function leverages the identity
$W(z\log z) = \log z$. Since $-\exp(-1)<-(4m)^{-1}<0$ it follows that $W$
is real- and strictly negative in value. Further, since
$\rho\leq 0.5<\exp(-1/2)\approx 0.6$, it follows that our solution lies again
in the lower branch as claimed.

To guarantee that the relation~\eqref{thm:main-rdp-condition} between $m, \gamma, \rho$
is still satisfied, we can solve the bound on $m$ in terms of $\gamma$:
\begin{eqnarray}
	m &\geq& \frac{1}{2(\gamma - \rho)^2}\log\left(\frac{1}{\rho}\right) \nonumber \\
	\Leftrightarrow \gamma &\geq& \rho + \sqrt{\frac{1}{2m}\log\left(\frac{1}{\rho}\right)}\enspace. \label{eq:solved-gamma}
\end{eqnarray}
Operating with this $\gamma$ establishes all the conditions of Theorem~\ref{thm:main-rdp}.
\end{proof}

\subsection{Fixed $\boldsymbol{m}$ Minimum $\boldsymbol{\gamma}$}
%\label{app:optimising-gamma}

\begin{cor}\label{cor:min-gamma}
For given fixed sampling resource budget $m\in\naturals$, taking
\begin{eqnarray*}
\rho &=& \exp\left(\frac{1}{2}W_{-1}\left(-\frac{1}{4m}\right)\right)\enspace, \\ %\label{eq:opt-rho-k} 
\gamma &=& \rho + \sqrt{\log(1/\rho)/(2m)}\enspace, \\ %\label{eq:opt-gamma-k}
k &=& m \enspace, %\label{eq:setting-k} 
\end{eqnarray*}
minimises privacy confidence parameter $\gamma$, when running Algorithm~\ref{algo:mechanism} to achieve 
$(\epsilon,\delta,\gamma)$-RDP.
\end{cor}

\begin{proof}
Consider now choosing $\rho$ to minimise $\gamma$, for given fixed $m$ sample
size budget, while then taking order statistic index $k$ according to the selected
$m, \rho, \gamma$. This corresponds to optimising the expression~\eqref{eq:solved-gamma}
with respect to $\rho$.
Noting that this expression is identical to the objective~\eqref{eq:tilde-k}, again
the global optimiser must be $\rho^\star=\exp(W_{-1}(-1/(4m))/2)$. With this choice of 
$\gamma$, the necessary $k$ equates to $m$.
\end{proof}

\begin{figure}[t]
\begin{center}
\centerline{\includegraphics[width=1.00\columnwidth]{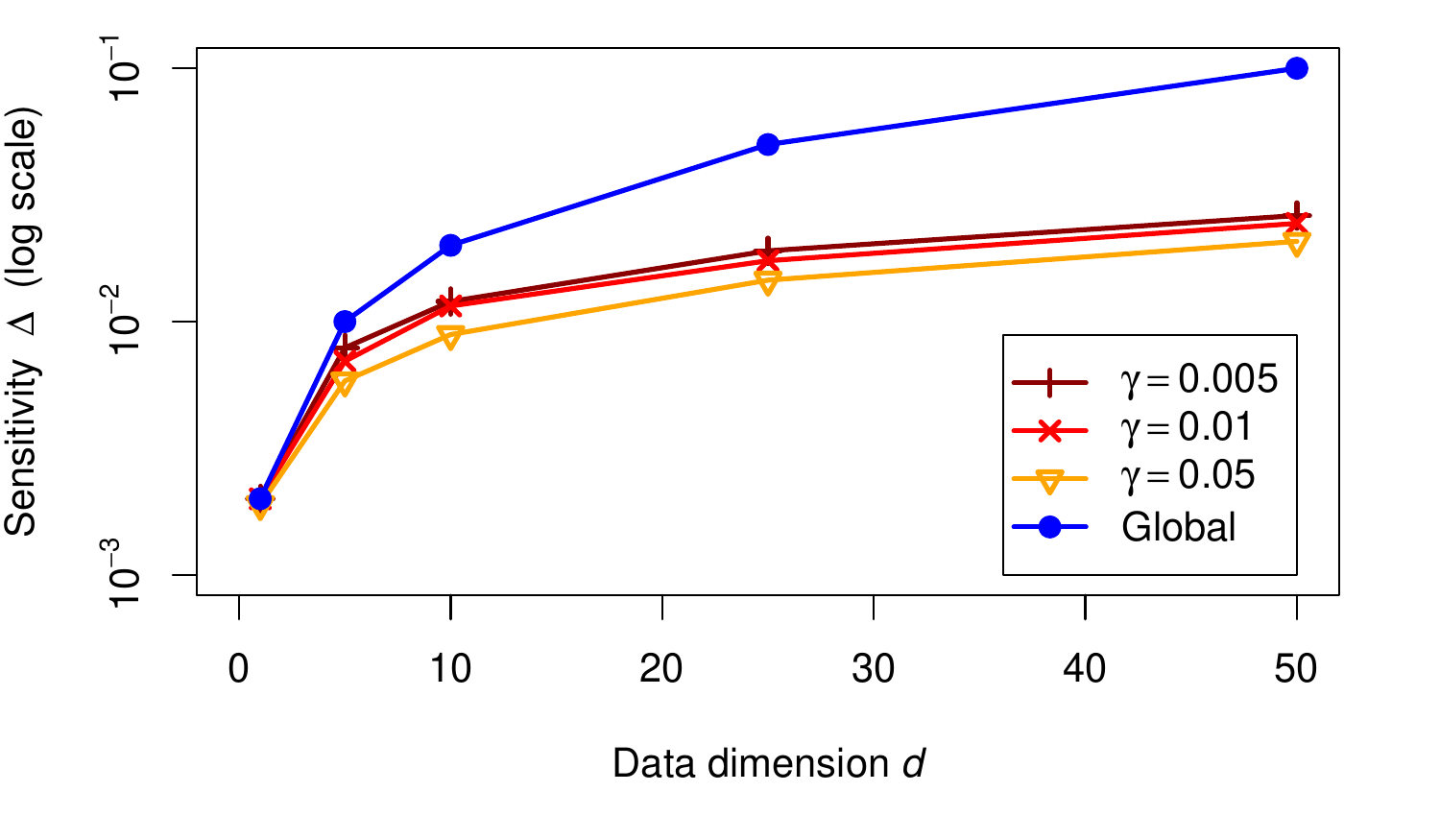}}
\caption{Global vs estimated sensitivity for the sample mean on bounded data.}
\label{fig:glob_vs_sampled_final}
\end{center}
\end{figure}

\section{Global vs. Sampled Sensitivity: Sample Mean of Bounded Data}\label{sec:l1-sensitivity}

Consider the goal of releasing the
sample mean $f(D)=n^{-1}\sum_{i=1}^n D_i$ of a database $D$ as in
Example~\ref{ex:sample-mean}, but 
over domain $\domain=[0,1]^d$. Figure~\ref{fig:glob_vs_sampled_final}
presents: the (sharp) bound on global sensitivity for this target
for use in \eg the Laplace mechanism; and the sensitivity $\hat{\Delta}$ estimated
by \sampler. Here $D$ comprises $n=500$ points sampled from the uniform
distribution over \domain, with \sampler run with optimised $m$ under
varying $\gamma$ as displayed. The reduction in sensitivity due to sampling is
striking (note the log scale). This experiment demonstrates sensitivity for different privacy guarantees (DP vs. RDP). By contrast for the same level of privacy (RDP) in Section~\ref{sec:analytical-vs-sampled}, \sampler quickly approaches the analytical approach.

\section{Proof of Proposition~\ref{prop:SVM}}
\label{sec:SVM-proof}

It follows immediately that $L=1$ and $\kappa=\sqrt{d}$. From the solution
$b=y_i-\sum_{j=1}^n \alpha_j y_j k(D_i, D_j)$ for some $i\in[n]$, combined
with the box constraints $0\leq\alpha_j\leq C/n$, the sensitivity of the bias
can be bounded as $2+2 C \sqrt{d}$. Combining with the existing normal vector
sensitivity yields the result.

\fi

\end{document}